%% file: main.tex
\pdfoutput=1
\documentclass[10pt]{scrartcl}

\input{config}

\begin{document}

\title{Warped-Linear Models for Time Series Classification}

\author{Brijnesh J.~Jain\\
 Technische Universit\"at Berlin, Germany\\
 e-mail: brijnesh.jain@gmail.com}
\date{}
\maketitle

\begin{abstract}
This article proposes and studies warped-linear models for time series classification. The proposed models are time-warp invariant analogues of linear models. Their construction is in line with time series averaging and extensions of k-means and learning vector quantization to dynamic time warping (DTW) spaces. The main theoretical result is that warped-linear models correspond to polyhedral classifiers in Euclidean spaces. This result simplifies the analysis of time-warp invariant models by reducing to max-linear functions. We exploit this relationship and derive solutions to the label-dependency problem and the problem of learning warped-linear models. Empirical results on time series classification suggest that warped-linear functions better trade solution quality against computation time than nearest-neighbor and prototype-based methods. 
\end{abstract}

\newpage

\tableofcontents
\clearpage

\section{Introduction}

Linear models are a mainstay of statistical pattern recognition. They make strong assumptions and yield stable but possibly inaccurate predictions \cite{Hastie2001}. Due to their simplicity and efficiency, they often serve as an initial trial classifier. In addition, linear models form the basis for techniques such as the perceptron, logistic regression, support vector machines, neural networks, and boosting \cite{Lebanon2005}. 

Linear models implicitly assume the geometry of Euclidean spaces. In non-Euclidean spaces, the theoretical insights on linear models and their implications break down. As a consequence, a mainstay of statistical pattern recognition is not available for non-Euclidean data. A powerful workaround to bridge this gap embeds distance spaces into linear spaces in either implicit or explicit manner.  Apart from feature extraction methods, common examples of such embeddings are kernel methods \cite{Schoelkopf2002} and dissimilarity representations \cite{Pekalska2005}. 

Such embeddings generally fail to preserve proximity relations and do not contribute much towards a better understanding of the nature of the original distance space, which is helpful for constructing more sophisticated classifiers. This holds, in particular, for time series endowed with the dynamic time warping (DTW) distance \cite{Sakoe1978}. Time series classification finds applications in diverse domains such as speech recognition, medical signal analysis, and recognition of gestures \cite{Fu2011,Geurts2001}. Notable, the prime approach in time series classification is the simple kNN method in conjunction with the DTW distance \cite{Bagnall2017}. In contrast, linear models for time series data do not play a significant role. The Euclidean geometry fails to filter out variations in temporal dynamics, which often leads to models that poorly fit the data.

\medskip 

Departing from the work on elastic classifiers \cite{Jain2015}, this article studies time-warp-invariant analogues of linear models. The contributions are as follows:

\medskip 

\noindent
\emph{Warped-linear functions.}\ We propose warped-linear functions that comprise warped-product and elastic-product functions as two techniques to enhance linear models with time-warp invariance. Warped-product functions replace the inner product between feature vectors by a warped product between time series. Warped products can be regarded as a similarity measure dual to the DTW distance. Elastic-product functions as proposed in \cite{Jain2015} replace the inner product by the product between a weight matrix and an input time series along a warping path. Here, we propose and analyze a slightly modified but substantially more flexible variant of elastic-product functions. Construction of warped-linear functions is in line with time series averaging \cite{Hautamaki2008,Petitjean2011,Schultz2017} and time-warp invariant extensions of k-means  \cite{Petitjean2016,Rabiner1979}, self-organizing maps \cite{Kohonen1998}, and learning vector quantization \cite{Jain2018,Somervuo1999}. 

\medskip 

\noindent
\emph{Equivalence to polyhedral classifiers.}\ 
Theorem \ref{theorem:warped=max-linear} states that warped-linear classifiers are equivalent to polyhedral classifiers under mild assumptions. Polyhedral classifiers are piecewise linear classifiers whose negative class region forms a convex polyhedron (see Figure \ref{fig:ex_separable}). This result is useful, because it simplifies the analysis of warped-linear functions by studying max-linear functions as handled in the next two contributions.  

\medskip 

\noindent
\emph{Label dependency.}\ 
We show that warped-linear classifiers are label-dependent and present a simple solution for this problem. Label dependency means that separability of two sets depends on how these sets are labeled as positive and negative. As an example, consider the second and third plot of Figure \ref{fig:ex_separable}. Both plots label the convex set $\S{U}$ as negative and the non-convex set $\S{V}$ as positive. In both cases we can find a polyhedron that contains the negative set $\S{U}$ and is disjoint to the positive set $\S{V}$. Hence both sets can be separated by warped-linear classifiers due to their equivalence to polyhedral classifiers. 
Now suppose that we would flip the labels such that $\S{U}$ is the positive and $\S{V}$ is the negative set. Then it is impossible to find a polyhedron that contains the negative set $\S{V}$ and is disjoint to the positive set $\S{U}$. Consequently, both sets can not be separated by a warped-linear classifier. A solution to this problem is using two instead of one discriminant function, one for each class.

\begin{figure}[t]
\centering
 \includegraphics[width=0.9\textwidth]{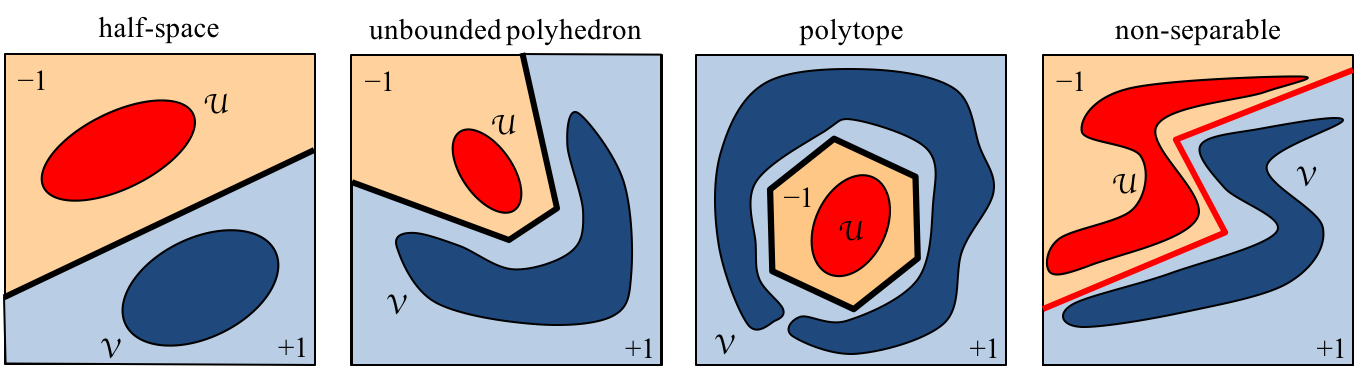}
\caption{The first three plots show two sets $\S{U}$ (red) and $\S{V}$ (blue) that are  separable by a polyhedral classifier. The first plot shows that polyhedral classifiers generalize linear classifiers. The second and third plot show an unbounded and bounded polyhedron, resp., as negative class region. The fourth plot shows two sets that can not be separated by a polyhedral classifier. }
\label{fig:ex_separable} 
\end{figure}

\medskip 

\noindent
\emph{Subgradient Method.}\ 
The regularized empirical risk over warped-linear functions is non-differentiable, because warped-linear functions are non-differentiable. By resorting to max-linear functions, we present a stochastic subgradient method for minimizing the regularized empirical risk under the assumption of a convex loss function.

\medskip 

\noindent
\emph{Experiments.}\ 
Empirical results in time series classification suggest that warped-linear classifiers complement nearest-neighbor and prototype-based methods in DTW spaces with respect to solution quality. In addition, we found that elastic-product classifiers were most efficient by at least one order of magnitude.

\medskip

The rest of this article is structured as follows: The remainder of this section discussed related work. Section 2 introduces warped-linear classifiers. Section 3 analyzes warped-linear classifiers via their relationship to max-linear classifiers and formulates a stochastic subgradient method for learning. Section 4 present experiments. Finally, Section 5 concludes with a summary of the main results and an outlook on further research. All proofs are delegated to the appendix.  

\subsection*{Related Work}

We discuss work related to elastic-product and warped-product functions as well as polyhedral classifiers. 

\paragraph*{Elastic-Product Functions}
The present work is based on elastic methods proposed by \cite{Jain2015}. Elastic methods combine dynamic time warping and gradient-based learning to extend standard learning methods such as linear classifiers, artificial neural networks, and k-means to warped time series. As main contribution, the work on elastic methods presents a unifying theoretical framework for adaptive methods in DTW spaces. As a proof of concept of this generic framework, elastic-product classifiers have been tested on two-category problems, only. 
The main contributions compared to \cite{Jain2015} are as follows:
\begin{enumerate}
\itemsep0em
\item Relationship to polyhedral classifiers.
\item Label dependency.
\item Subgradient methods.
\item Experiments on multi-category problems. 
\end{enumerate}
To prove Theorem \ref{theorem:warped=max-linear} for establishing the equivalence to polyhedral classifiers, we modify the original elastic-product functions proposed in \cite{Jain2015} in two ways: (i) we replace the bias by an elastic bias and (ii) we allow restrictions to subsets of warping paths. 

Learning the original elastic-product functions has been framed within the more general setting of elastic methods and amounts in minimizing the empirical risk by stochastic generalized gradient methods in the sense of Norkin \cite{Norkin1986}. Here, we render the 
stochastic generalized gradient method more precisely as a stochastic subgradient method.

\paragraph*{Warped-Product Functions}
The idea of using warped-product functions as a substitute of the inner product is in line with similar approaches that replace the Euclidean distance by the DTW distance in order to extend adaptive learning methods from Euclidean spaces to DTW spaces. Examples include time series averaging \cite{Abdulla2003,Cuturi2017,Hautamaki2008,Kruskal1983,Rabiner1979,Rabiner1980,Schultz2017,Petitjean2011}, k-means clustering \cite{Hautamaki2008,Niennattrakul2007a,Niennattrakul2007b,Petitjean2016,Rabiner1979,Rabiner1980,Soheily-Khah2016,Wilpon1985}, self-organizing maps \cite{Kohonen1998,Somervuo1999}, and learning vector quantization \cite{Somervuo1999,Jain2018}. Warped-product function have been mentioned in \cite{Jain2015} in a more general setting but rejected as less flexible than elastic-product functions. Theorem \ref{theorem:warped=max-linear} partly refutes this claim.

\paragraph*{Polyhedral Classifiers}
Theorem \ref{theorem:warped=max-linear} shows that enhancing linear models by time-warp invariances results in polyhedral classifiers under mild assumptions. Polyhedral classifiers and polyhedral separability have been studied for three decades \cite{Astorino2002,Dundar2008,Kantchelian2014,Manwani2010,Megiddo1988,Orsenigo2007,Takacs2009,Yujian2011}. None of the proposed methods considered a formulation inspired by time-warp invariances. To understand the effect of time-warp invariance, we compare warped linear functions with a standard polyhedral classifier closely related to the approach proposed by \cite{Manwani2010}.

\section{Warped-Linear Functions}

Warped-linear functions extend linear functions to time series spaces using the concept of dynamic time warping. This section introduces two approaches: warped-product and elastic-product functions. 

\subsection{Dynamic Time Warping}

A time series $x$ of length $l(x) = n$ is a sequence $x = (x_1, \ldots, x_n)$ consisting of elements $x_i \in \R$ for every time point $i \in [n] = \cbrace{1, \ldots, n}$. We use $\R^n$ to denote the set of time series of length $n$. Then 
\[
\S{T}_n = \bigcup_{d = 1}^n \R^d
\]
is the set of all time series of bounded length $n$ and $\S{T}$ is the set of time series of finite length. 

Different time series representing the same concept can vary in length and speed. A common technique to cope with these variations is dynamic time warping (DTW). Dynamic time warping aligns (warps) time series by locally compressing and expanding their time segments. Such alignments are described by warping paths. 

To define warping paths, we consider a ($m \times n$)-lattice $\S{L}_{m,n}$ consisting of $m \cdot n$ points at the intersections of $m$ horizontal and $n$ vertical lines. A \emph{warping path} in lattice $\S{L}_{m,n}$ is a sequence $p = (p_1 , \dots, p_L)$ of $L$ points $p_l = (i_l,j_l) \in \S{L}_{m,n}$ such that
\begin{enumerate}
\item $p_1 = (1,1)$ and $p_L = (m,n)$ \hfill(boundary conditions)
\item $p_{l+1} - p_{l} \in \cbrace{(1,0), (0,1), (1,1)}$ for all $l \in [L-1]$ \hfill (step condition)
\end{enumerate}
By $\S{P}_{m,n}$ we denote the set of all warping paths in $\S{L}_{m,n}$. A warping path departs at the upper left corner $(1,1)$ and ends at the lower right corner $(m,n)$ of the lattice. Only west $(1, 0)$, south $(0, 1)$, and southwest $(1, 1)$ steps are allowed to move from a given point $p_l$ to the next point $p_{l+1}$ for all $1 \leq l < L$. A point $(i,j)$ of warping path $p$ means that time point $i$ of the first time series is aligned to time point $j$ of the second time series. We write $(i,j) \in p$ to denote that $(i,j)$ is a point in path $p$.  

\subsection{Warped-Product Functions}

\begin{figure}[t]
\centering
 \includegraphics[width=0.8\textwidth]{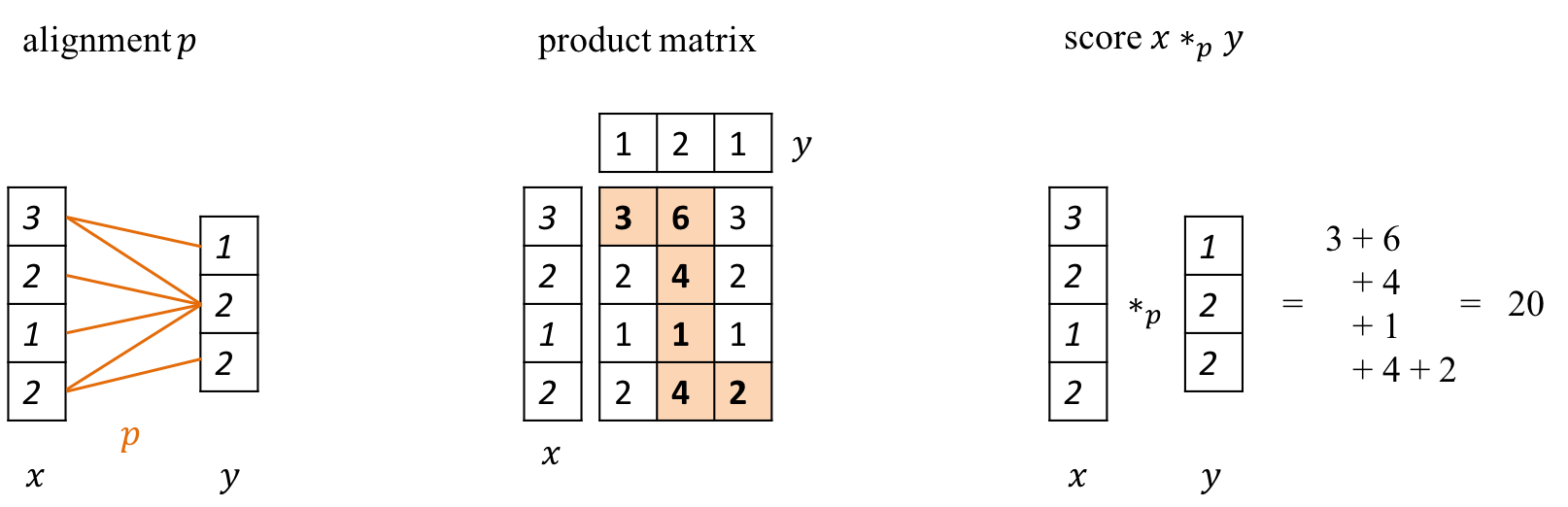}
\caption{Example of a score $x \ast_p y$ of two time series $x$ and $y$ defined by a warping path $p$. Since $p$ is an optimal warping path, the score $x \ast_p y$ coincides with the warped-product $x \ast y$. \emph{Left:} Two time series $x$ and $y$ shown as columns. Orange lines connecting the elements of $x$ and $y$ indicate the warping path $p$. \emph{Middle:} Local product matrix showing all possible products $x_i\cdot x_j$. The product along warping path $p$ are highlighted. \emph{Right:} Score $x \ast_p y = 20$ as sum of highlighted products along warping path $p$.}
\label{fig:ex_wp} 
\end{figure}

The first of two approaches to extend linear functions to time series spaces are warped-product functions. The basic idea is to replace the inner product between vectors by a warped product between time series. Warped products can be regarded as a similarity measure dual to the DTW distance \cite{Sakoe1978}.

\medskip

Suppose that $x, y \in \S{T}$ are two time series of length $l(x) = m$ and $l(y) = n$, respectively. Every warping path $p \in \S{P}_{m,n}$ defines a score 
\[
x \ast_p y = \sum_{(i,j) \in p} x_i y_j
\]
of aligning $x$ and $y$ along path $p$. The \emph{warped-product} is a function $\ast: \S{T} \times \S{T} \rightarrow \R$ defined by
\[
x \ast y = \max \cbrace{x \ast_p y \,:\, p \in \S{P}_{m,n}}.
\]
An \emph{optimal warping path} for $x \ast y$ is any path $p \in \S{P}_{m,n}$ such that $x \ast y = x \ast_p y$. Figure \ref{fig:ex_wp} illustrates the definitions of a score $x \ast_p y$ and warped-product $x \ast y$. 

A \emph{warped-product function} is a function of the form
\[
f: \S{T} \rightarrow \R, \quad x \mapsto (b, w) \ast (1, x).
\]
where $w \in \S{T}$ is the \emph{weight sequence} and $b \in \R$ is the \emph{elastic bias} of $f$. For convenience, we will use a more compact notation for warped-product functions.
\begin{notation}
By $\S{X} = \cbrace{1} \times \S{T}$ we denote the \emph{augmented input space}. Then a warped-product function $f$ can be written as 
\[
f: \S{X} \rightarrow \R, \quad x \mapsto w \ast x,
\]
where $w$ is the augmented weight sequence that includes the bias as first element. 
\end{notation}

Suppose that $f(x) = w \ast x$ is warped-product function. The length $e = l(w)$ of the (augmented) weight sequence $w$ is a hyper-parameter, called \emph{elasticity} of $f$ henceforth.

\subsection{Elastic-Product Functions}

The second of two approaches to extend linear functions to time series spaces are warped-product functions. 

\medskip

Elastic products warp time series into a matrix along a warping path. To define the score of such a warping we proceed in two steps. The first steps assumes time series of fixed length $d$. The second step generalizes to time series of bounded length $d$.

Let $W \in  \R^{d \times e}$ be a matrix. Suppose that $x \in \R^d$ is a time series of length $d$. Then every warping path $p \in \S{P}_{d, e}$ defines the score 
\[
W \otimes_p x = \sum_{(i,j) \in p} w_{ij} x_i.
\]
To extend the score $W \otimes_p x$ to time series $x \in \S{T}_d$ of bounded length $l(x) \leq d$, we write $W_{[1:l(x)]}$ to denote the sub-matrix consisting of the first $l(x)$ rows of matrix $W$. Then every warping path $p \in \S{P}_{l(x), e}$ defines the score 
\[
W \otimes_p x = W_{[1:l(x)]} \otimes_p x.
\]
Figure \ref{fig:ex_ep} illustrates the score $W \otimes_p x$ defined by a warping path $p$. An \emph{elastic product} is a function $\otimes: \R^{d \times e} \times \S{T}_d \rightarrow \R$ defined by
\[
W \otimes x = \max \cbrace{W \otimes_p x \,:\, p \in \S{P}_{l(x),e}}.
\] 
An \emph{optimal warping path} for $W \otimes x$ is any path $p \in \S{P}_{m,n}$ such that $W \otimes x = W \otimes_p x$. An \emph{elastic-product function} of elasticity $e$ is a function of the form 
\[
f: \S{T}_d \rightarrow \R, \quad x \mapsto \begin{pmatrix}
{b}\T\\
W
\end{pmatrix} 
\otimes 
\begin{pmatrix}
1\\
x
\end{pmatrix}, 
\]
where $W \in \R^{d\times e}$ is the \emph{weight matrix} and the vector $b \in \R^e$ is the \emph{elastic bias} of $f$. We write  elastic-product functions in augmented form but reduce the dimension to keep the notation simple. 
\begin{notation}
Let $d > 1$.By $\S{X} = \cbrace{1} \times \,\S{T}_{d-1}$ we denote the augmented input space. Then an elastic-product function can be written as
\[
f: \S{X} \rightarrow \R, \quad x \mapsto W \otimes x,
\]
where $W \in \R^{d\times e}$ is the augmented weight matrix containing the elastic bias in its first row. 
\end{notation}

\begin{figure}[t]
\centering
 \includegraphics[width=0.45\textwidth]{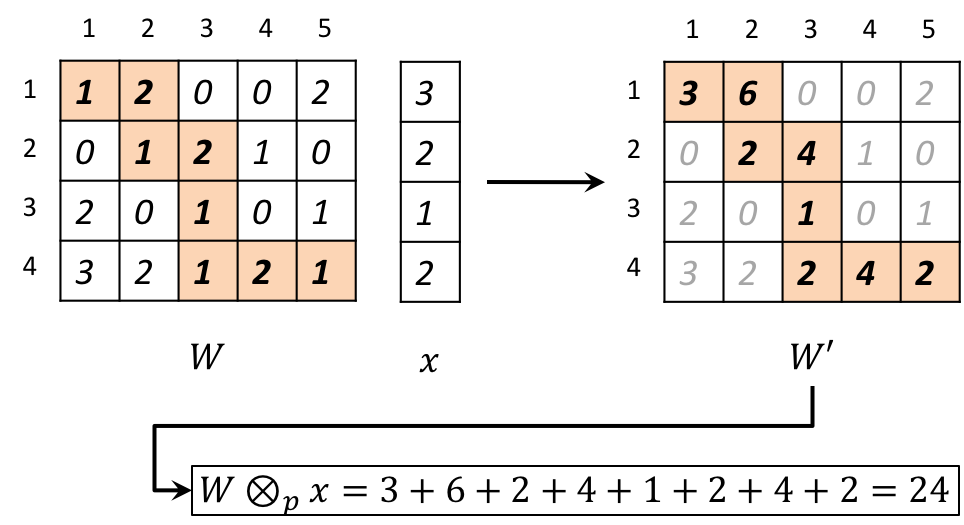}
 \hfill
 \includegraphics[width=0.45\textwidth]{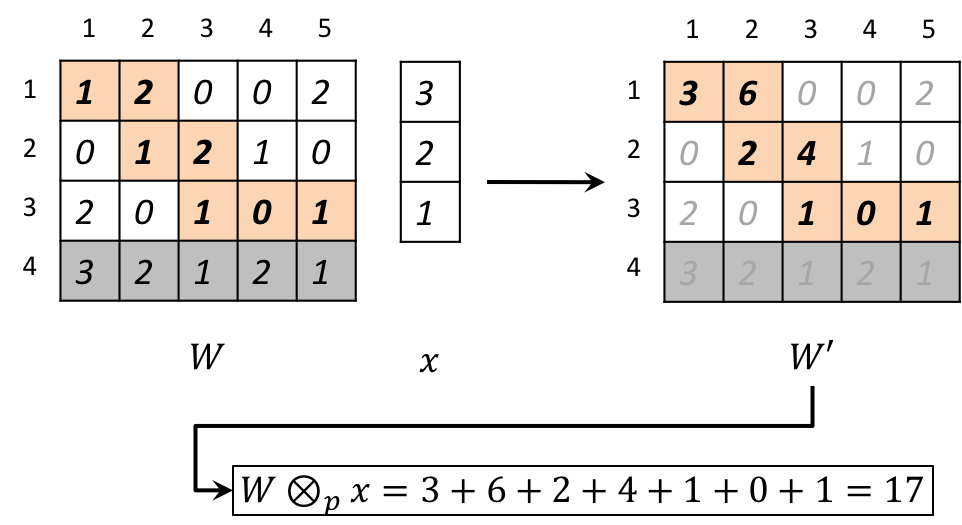}
\caption{Examples of scores $W \otimes_p x$ defined by a warping path $p$. Both examples use the same weight matrix $W \in \R^{4 \times 5}$ with length (number of rows) $4$ and elasticity $5$. The matrix $W$ highlights the respective warping paths $p$ by orange cells and bold-faced numbers. \emph{Left:} $W \otimes_p x$ for a time series $x$ of the same length $4$ as $W$. Since $x$ has length $4$, the warping path $p$ is from the set $\S{P}_{4,5}$. The matrix $W'$ highlights the products $w_{ij} x_i$ for all $(i,j) \in p$. Summing the highlighted products gives the score $W \otimes_p x = 27$. \emph{Right:} $W \otimes_p x$ for a time series $x$ of length $3 < 4$. In this case, the warping path $p$ is from $\S{P}_{3,5}$. Thus, the fourth row in $W$ is ignored such that this case reduces to the case of the left figure when the length of the matrix and the time series coincide.}
\label{fig:ex_ep} 
\end{figure}

\subsection{Warping Constraints}

This section completes the definition of warped-linear functions by imposing warping constraints. Warping constraints restrict the set of admissible warping paths to some nonempty subset. Such restrictions have been originally introduced to improve performance of applications based on the DTW distance. Popular examples are the Sakoe-Chiba band \cite{Sakoe1978} and the Itakura parallelogram \cite{Itakura1975}. Here, we use warping constraints to achieve theoretical flexibility. 

\medskip

Suppose that $x$ and $y$ are time series of length $\ell(x) = m$ and $\ell(y) = n$. Let $\S{Q} \subseteq \S{P}_{m,n}$ is a subset. The 
warped product between $x$ and $y$ in (or constrained to) $\S{Q}$ is of the form
\[
x \ast y = \max \cbrace{x \ast_p y\,:\, p \in \S{Q}}
\]
Similarly, we say that 
\[
W \otimes x = \max \cbrace{W \otimes_p x \,:\, p \in \S{Q}_{l(x)}}.
\]
is the elastic product between $W$ and $x$ in (or constrained to) $\S{Q} = \bigcup_{l=1}^d \S{Q}_l$, where $\S{Q}_{l} \subseteq \S{P}_{l,e}$ is a subset for every $l \in [d]$. In the same manner, we define optimal warping paths in $\S{Q}$ and warped-linear functions in $\S{Q}$.

\section{Max-Linear Models for Time Series Classification}

The definition of warped-linear functions follows the traditional problem-solving approach of dynamic time warping and optimal sequence alignment. The traditional approach is well suited for translation into dynamic programming solutions but less suited for analytical purposes. In this section, we suggest max-linear functions as a more suitable representation for studying warped-linear functions. The analysis rests on the following assumption:

\bigskip

\hrule
\begin{assumption}\label{assumption:general}
We analyze warped-linear functions length-wise. This means, the following results hold for subspaces of time series of identical length. For this, we assume that $\S{X}$ is an augmented input space of the form $\S{X} = \cbrace{1} \times \R^{d-1} \subseteq \R^d$. 

\medskip

A warped-linear function is always constrained to some (not necessarily proper) subset of all possible warping paths given the input dimension and elasticity. Thus, warped-linear functions subsumes constrained as well as unconstrained warped-linear functions. We use the following notations:
\begin{enumerate}
\itemsep0em
\item $\S{W}_d$ is the set of all (constrained and unconstrained) warped-product functions on $\S{X}$ of finite elasticity
\item $\S{E}_d$ is the set of all (constrained and unconstrained) elastic-product functions on $\S{X}$ of finite elasticity.
\end{enumerate}
\end{assumption}

\begin{remark}
Restriction to augmented time series of fixed length $d$ serves a better theoretical understanding of warped-linear functions but imposes no practical limitations. In a practical setting, we admit time series of varying length. 
\end{remark}
\hrule

\subsection{Max-Linear Functions}

This section shows that warped-linear functions are pointwise maximizers of linear functions. This result simplifies the study of analytically complicated functions based on dynamic time warping to the study of a much simpler class of functions.

\medskip

A \emph{generalized linear function} is a function of the form 
\[
f: \R^d \rightarrow \R, \quad  x \mapsto w\T \phi(x) + b,
\]
where $\phi: \R^d \rightarrow \R^m$ is a linear transformation, $w \in \R^m$ is the weight vector, and $b \in \R$ is the bias. If $\phi = \id$ is the identity, we recover the definition of standard linear functions. Note that generalized linear functions can be equivalently expressed as standard linear functions. Here, the notion of generalized linear function serves to emphasize that the weight vector can be of different dimension than the input vector.  

\medskip

To be consistent with warped-linear functions, we represent generalized linear functions in augmented form but reduce the dimension to keep the notation simple. 
\begin{notation}
Let $\S{X}$ be the augmented input space of Assumption \ref{assumption:general}. We consider linear transformations $\phi: \R^d \rightarrow \R^m$ that satisfy $\phi(\S{X}) \subseteq \cbrace{1} \times \R^{m-1}$. Then a generalized linear function can be written in compact form as 
\[
f: \S{X} \rightarrow \R, \quad  x \mapsto w\T \phi(x),
\]
where $w \in \R^m$ is the augmented weight vector that contains the bias $b$ as its first element. 
\end{notation}

Suppose that $\S{P} \neq \emptyset$ is a finite index set. A \emph{max-linear function} is a function defined by 
\[
f:\S{X} \rightarrow \R, \quad x \mapsto \max \cbrace{f_p(x) \,:\, p \in \S{P}},
\]
where the components $f_p(x) = w_p{\T}\phi_p(x)$ are generalized linear functions. The \emph{active set} of a max-linear function $f$ at point $x \in \R^d$ is the subset defined by
\[
\S{A}_f(x) = \cbrace{f_p \,:\; f_p(x) = f(x), p \in \S{P}}.
\] 
The elements $f_p \in \S{A}_f(x)$ are the \emph{active components} of $f$ at $x$. By $\S{L}_d$ we denote the set of all max-linear functions on $\R^d$ with finite number of component functions.

\medskip

We present the main result of this contribution. It shows that warped-linear functions are max-linear. 

\begin{theorem}\label{theorem:warped=max-linear}
$\S{W}_d \subseteq \S{E}_d = \S{L}_d$.
\end{theorem}

For $d = 1$ we can construct a max-linear function without bias that can not be expressed as a warped-product function. Using a slightly modified version $\S{W}'_d$ of  the set $\S{W}_d$, we can also show equality  $\S{W}'_d = \S{L}_d$. For this, we augment the (already augmented) input space $\S{X}$ by a leading and trailing zero for warped-product functions.

\begin{proposition}\label{prop:W'_d = L_d}
Let $\S{W}'_d$ be the set of warped-product functions on $\S{X}' = \cbrace{0} \times \S{X} \times \cbrace {0}$ with finite elasticity. Then $\S{W}'_d = \S{L}_d$. 
\end{proposition}

In the remainder of this section, we show how warped-linear functions are related to max-linear functions.

\subsubsection*{Warped-Product Functions}

To express warped-product functions as max-linear function, we introduce warping matrices. Let $w$ be a weight sequence of length $l(w) = e$, let $x \in \S{X}$ be a time series of length $l(x) = d$, and let $Q \subseteq \S{P}_{e, d}$ be a subset. The \emph{warping matrix} of a given warping path $p \in \S{Q}$ is a binary matrix $M_p \in \{0,1\}^{e \times d}$ with elements
\begin{equation*}
m_{ij}^p = \begin{cases} 
1 & (i,j) \in p \\ 
0 & \text{otherwise} 
\end{cases}.
\end{equation*}
A warping matrix is an equivalent representation of a warping path. Consider the linear transformation
\[
\phi_p(x) = M_p\, x.
\]
As a consequence of Theorem \ref{theorem:warped=max-linear}, a warped-product function $f(x) = w \ast x$ can be equivalently written as 
\[
f(x) = \max\cbrace{f_p(x)= w\T\phi_p(x) \,:\, p \in \S{Q}}, 
\]
where the components $f_p(x)$ are generalized linear functions indexed by warping paths $p$. Thus, a component function $f_p$ is active at $x$ if $p$ is an optimal warping path for $w \ast x$.

\subsubsection*{Elastic-Product Functions}

To express elastic-product functions as max-linear functions, we assume that $W \in \R^{d \times e}$ is a weight matrix, $x \in \S{X}$ is a time series, and $Q \subseteq \S{P}_{d,e}$ is a subset. Every warping path $p \in \S{Q}$ induces a map
\[
\phi_p: \S{X} \mapsto \R^{d \times e}, \quad x \mapsto X_p,
\]
where $X_p = \args{x_{ij}^p}$ is the \emph{$p$-matrix} of $x$ with elements
\[
x_{ij}^p = \begin{cases}
x_i & (i,j) \in p\\
0 & \text{otherwise}
\end{cases}.
\]
The $p$-matrix $X_p$ represents $x$ in the lattice $\S{L}_{d,e}$ along warping path $p$. Lemma \ref{lemma:x->Xp} shows that the map $\phi_p$ is  linear. Consider the Frobenius inner product defined by
\[
\inner{W, \phi_p(x)} = \inner{W, X_p} = \sum_{i=1}^d \sum_{j=1}^e w_{ij}\,x_{ij}^p.
\]
The Frobenius inner product is an inner product between matrices as though they are vectors and is therefore a generalized linear function. As a consequence of Theorem \ref{theorem:warped=max-linear}, an elastic-product function $f(x) = W \otimes x$ is of the form
\[
f(x) = \max \cbrace{f_p(x) = \inner{W, X_p} \,:\, p \in \S{Q}},
\] 
where the components $f_p(x)$ are generalized linear functions indexed by warping paths $p$. Thus, a component function $f_p$ is active at $x$ if $p$ is an optimal warping path for $W \otimes x$.  

\subsection{Discriminant Functions}

In this section, we use max-linear functions as discriminant functions for representing classifiers. We show that a single max-linear discriminant function for the two-category case is not justified. This results is in contrast to the traditional treatment of the two category case using linear discriminant functions. 

\medskip

Suppose that $\S{X} = \R^d$ is the input space and  $\S{Y} = \cbrace{1, \ldots, K}$ is a set consisting of $K$ class labels. Consider $K$ max-linear discriminant functions $f_1, \ldots, f_K:\S{X} \rightarrow \S{Y}$. Then the decision rule based on $K$ discriminant functions $f_k$ classifies points $x \in \S{X}$ to class labels $y(x) \in \S{Y}$ according to the rule\footnote{The $\argmax$-operation may return a subset $\S{Y}' \subseteq \S{Y}$. If $\abs{\S{Y}'} > 1$ we pick a random element from $\S{Y}'$.}  
\[
y(x) \in \argmax_{k \in \S{Y}} f_k(x).
\]
The two-category case is just a special instance of the multicategory case, but has traditionally received a separate treatment \cite{Duda2001}. For max-linear discriminant functions such a separate treatment can lead into a pitfall. To see this, we assume that $K = 2$. Then the decision rule reduces to 
\[
y(x) = \begin{cases}
1 & f_1(x) > f_2(x)\\
2 & \text{otherwise}
\end{cases}.
\]
Hence, we can use a single discriminant function $f = f_1 - f_2$ to obtain the equivalent decision rule that assigns $x$ to class $y(x) = 1$ if $f(x) > 0$ and to class $y(x) = 2$ otherwise. In contrast to standard linear classifiers, the difference of max-linear functions is generally not a max-linear function. Thus, using a single max-linear functions for the two-category case is not justified. To show this, consider the sets
\begin{align*}
\Delta(\S{L}_d) &= \cbrace{f_1 - f_2 \,:\ f_1,f_2 \in \S{L}_d}\\
\linear\args{\S{L}_d} &= \cbrace{\lambda_1 f_1 +  \lambda_2 f_2 \,:\ f_1,f_2 \in \S{L}_d,\: \lambda_1, \lambda_2 \in \R}.
\end{align*}
We call $\Delta(\S{L}_d)$ the \emph{difference hull} and $\linear\args{\S{L}_d}$ the \emph{linear hull} of $\S{L}_d$. We have the following relationships between the sets:
\begin{proposition}\label{prop:L C diff C lin}
$\S{L}_d \subsetneq \Delta(\S{L}_d) = \linear\args{\S{L}_d}$.
\end{proposition}
The implications of Prop.~\ref{prop:L C diff C lin} are twofold: (i) a single max-linear discriminant function for the two-category case is not justified, because $\S{L}_d \subsetneq \Delta(\S{L}_d)$; and (ii) the difference closure $\Delta(\S{L}_d) = \linear(\S{L}_d)$ is more expressive than the set $\S{L}_d$. Consequently, the decision rule in the two-category case should be based on two instead of a single max-linear discriminant functions. In other words, the two-category case should be treated in the same way as the multicategory case.

\medskip

From a geometric point of view, Prop.~\ref{prop:L C diff C lin} has the following implication: Every discriminant function $f \in \linear\args{\S{L}_d}$ defines a \emph{decision surface} 
\[
\S{S}(f) = \cbrace{x \in \S{X} \,:\, f(x) = 0}
\]
that partitions the input space $\S{X}$ into two regions 
\[
\S{R}_f^+ = \cbrace{x \in \S{X} \,:\, f(x) > 0} \quad \text{ and } \quad \S{R}_f^- = \cbrace{x \in \S{X} \,:\, f(x) \leq 0}.
\]
The open set $\S{R}_f^+$ is the region for the positive class and the closed set $\S{R}_f^-$ is the region for the negative class. Then from Prop.~\ref{prop:L C diff C lin} follows that the difference hull $\Delta(\S{L}_d)$ can implement more decision surfaces and therefore separate more class regions than the set $\S{L}_d$ of max-linear functions.

\subsection{Max-Lin Separability}\label{subsec:separability}

The negative class region implemented by a single max-linear discriminant function is a convex polyhedron \cite{Astorino2002,Bagirov2014,Megiddo1988}. This result indicates that max-lin separability of two sets depends on how we label the classes as positive and negative. Label-independent separability can be achieved by using two instead of a single max-linear discriminant function. Due to Theorem \ref{theorem:warped=max-linear}, these results automatically carry over to warped-product and elastic-product discriminants.

\medskip

 We say, a set $\S{U} \subseteq \R^d$ is \emph{max-lin separable from} $\S{V} \subseteq \R^d$ if there is a max-linear function $f \in \S{L}_d$ such that 
\begin{enumerate}
\item $f(x) < 0$ for all $x \in \S{U}$
\item $f(x) > 0$ for all $x \in \S{V}$.
\end{enumerate}
Note that max-lin separability is equivalent to polyhedral separability introduced by \cite{Megiddo1988}. To present conditions under which two finite sets can be perfectly separated by a max-linear discriminant function, we use the notion of convex hull. Let $\S{U} = \cbrace{x_1, \ldots, x_K} \subseteq \R^d$ be a finite set. The \emph{convex hull} of $\S{U}$ is defined by
\[
\conv(\S{U}) = \cbrace{\lambda_1 x_1 + \cdots + \lambda_K x_K \,:\, \lambda_1 + \cdots +\lambda_K = 1 \text{ and } \lambda_1, \ldots, \lambda_K \in \R_{\geq 0}}.
\]
The next result presents a necessary and sufficient condition of max-lin separability.
\begin{proposition}\label{prop:L-separability}
Let $\S{U}, \S{V} \subseteq \R^d$ be two finite non-empty sets. Then $\S{U}$ is max-lin separable from $\S{V}$ if and only if 
\[
\conv(\S{U}) \; \bigcap \; \S{V} = \emptyset.
\]
\end{proposition}

\begin{proof}
\cite{Astorino2002}, Prop.~2.2. 
\end{proof}

\medskip

Proposition \ref{prop:L-separability} shows that max-lin separability is asymmetric in the sense that $\conv(\S{U}) \cap\S{V} = \emptyset$ does not imply $\S{U} \cap \conv(\S{V}) = \emptyset$. In other words, the statement that \emph{$\S{U}$ is max-lin separable from $\S{V}$} does not imply the statement that \emph{$\S{V}$ is max-lin separable from $\S{U}$}. As an example, we consider the problem of separating points inside a unit square $\S{U}$ from points outside the square $\S{V} = \overline{\S{U}}$. There are two ways to label the points $x \in \R^2$:
\[
y_1(x) \mapsto \begin{cases}
-1 & x \in \S{U}\\
+1 & x \in \S{V}
\end{cases}
\quad \text{ and } \quad
y_2(x) \mapsto \begin{cases}
-1 & x \in \S{V}\\
+1 & x \in \S{U}
\end{cases}
\]
The first labeling function $y_1(x)$ regards the unit square $\S{U}$ as the negative class region and the second labeling function $y_2(x)$ regards the complement $\S{V}$ of the unit square as the negative class region. Figure \ref{fig:ex_squares} shows an example of two finite subsets $\S{U}' \subseteq \S{U}$ and $\S{V}' \subseteq \S{V}$ such that $\conv(\S{U}') \cap \S{V}' = \emptyset$ and $\S{U}' \cap \conv(\S{V}') \neq \emptyset$. This implies that $\S{U}'$ is max-lin separable from $\S{V}$ but $\S{V}$ is not max-lin separable from $\S{U}$. Thus, we can only separate both sets $\S{U}'$ and $\S{V}'$ with a single max-linear discriminant function if we use the first labeling function $y_1(x)$. This shows that classifiers based on a single max-linear discriminant function are label dependent.

\begin{figure}[t]
\centering
 \includegraphics[width=0.5\textwidth]{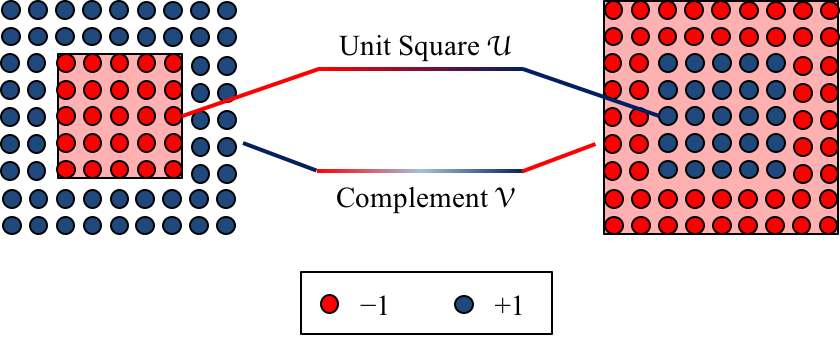}
\caption{Two different labelings of points from the unit square $\S{U}$ and its complement $\S{V}$. Points with negative (positive) class labels are shown in red (blue). The red shaded regions depict the convex hull of the negatively labeled points. The left figure shows that the intersection of the convex hull $\conv(\S{U})$ and $\S{V}$ is empty. The right figure shows that the convex hull $\conv(\S{V})$ includes the unit square $\S{U}$ as subset.}
\label{fig:ex_squares} 
\end{figure}

To obtain a label-independent classifier, we use two max-linear discriminant functions. As in the example, we assume that $\S{U}$ is max-linear separable from $\S{V}$ without loss of generality. We distinguish between the two alternatives to label the classes:
\begin{enumerate}
\item Labeling function $y_1(x)$:\\
By assumption, there is max-linear discriminant function $f$ that correctly separates the sets $\S{U}$ and $\S{V}$. The function $g = f - 0$ is the difference of two max-linear functions and trivially again a max-linear function that correctly separates both sets. 
\item Labeling function $y_2(x)$:\\
The discriminant $-f$ separates both sets $\S{U}$ and $\S{V}$, where $f$ is the max-linear discriminant from the first case. The discriminant $-f$ is not max-linear but the discriminant $g' = 0 - f$ is the difference of two max-linear functions that correctly separates both sets. 
\end{enumerate}
 
\medskip

Next, we show that the negative class region implemented by a max-linear discriminant function is a convex polyhedron. A \emph{polyhedron} $\S{P} \subseteq \S{X}$ is the intersection of finitely many closed half-spaces, that is
\[
\S{P}= \S{P}(A) = \cbrace{x \in \S{X} \,:\, Ax \leq 0},
\]
where $A \in \R^{m \times d}$ is a matrix. Recall that $\S{X}$ is an augmented input space such that the first column of matrix $A$ represents the bias vector. By $\Pi_{\S{X}}$ we denote the set of all polyhedra in $\S{X}$. The next result relates polyhedra and negative class regions defined by max-linear functions.

\begin{proposition}\label{prop:region=polyhedron} 
Let $2^{\S{X}}$ be the set of all subsets of $\S{X}$. The map
\[
\phi: \S{L}_d \rightarrow 2^{\S{X}}, \quad f \mapsto \S{R}_f^-
\]
satisfies $\phi\args{\S{L}_d} = \Pi_{\S{X}}$.
\end{proposition}

\begin{proof}
Follows from the equivalence of max-lin separability and polyhedral separability \cite{Megiddo1988}.
\end{proof}

Proposition \ref{prop:region=polyhedron} makes two statements: First, every negative class region implemented by a max-linear discriminant is a polyhedron; and second, every polyhedron coincides with the negative class region of some max-linear discriminant. Figure \ref{fig:ex_separable} presents examples of max-lin separable sets and an example of sets that are not max-lin separable.

\subsection{Learning}

This section presents a stochastic subgradient method for learning a max-linear discriminant function.

\subsubsection{Empirical Risk Minimization}

This section describes learning as the problem of minimizing a differentiable empirical risk function and presents a stochastic gradient descent rule. 

\medskip

Let $\S{Z} = \S{X} \times \S{Y}$ be the product of input space $\S{X}$ and output space $\S{Y}$. Consider the hypothesis space $\S{H}$ consisting of functions $f_\theta: \S{X} \rightarrow \S{Y}$ with adjustable parameters $\theta \in \R^q$. Suppose that we are given a training set $\S{D} = \cbrace{\args{x_1, y_1}, \ldots, \args{x_N, y_N}} \subseteq \S{Z}$ consisting of $N$ input-output examples $(x_i, y_i)$ drawn i.i.d.~from some unknown joint probability distribution on $\S{Z}$. According to the empirical risk minimization principle \cite{Vapnik1999}, learning amounts in finding a parameter configuration $\theta_* \in \R^q$ that minimizes the empirical risk 
\[
R_N(\theta) = \frac{1}{N}\sum_{i=1}^N \ell\args{y_i, f_\theta\!\args{x_i}},
\]
where $\ell: \S{Y} \times \R \rightarrow \R$ is a loss function that measures the discrepancy between the predicted output value $\hat{y}_i = f_\theta\!\args{x_i}$ and the actual output $y_i$ for a given input $x_i$. To improve generalization performance, one often considers a regularized empirical risk of the form 
\[
R_N^{\,\rho}(\theta) = R_N(\theta) + \lambda \rho(\theta),
\]
where $\lambda \geq 0$ is the regularization parameter and $\rho:\R^q \rightarrow \R$ is a non-negative regularization function. For $\lambda = 0$, we recover the standard empirical risk. Common choices of regularization functions are $\rho(\theta) = \normS{\theta}{^p}$ with $p = 1, 2$.

Let $(x, y) \in \S{D}$ be a training example. Suppose that the loss $\ell(y, f_\theta(x))$ and the regularization function $\rho(\theta)$ are both differentiable as functions of the parameter $\theta$. In this case, we can apply stochastic gradient descent to minimize the regularized empirical risk $R_N^{\,\rho}(\theta)$. The update rule of stochastic gradient descent is of the form
\begin{align}\label{eq:sgd}
\theta \; \leftarrow \; \theta - \eta \args{\ell' \,\nabla f_\theta(x) + \lambda \nabla \rho(\theta)},
\end{align}
where  $\eta > 0$ is the learning rate, $\ell'$ is the derivative of the loss $\ell(y, \hat{y})$ at the predicted output value $\hat{y} = f_\theta(x)$, $\nabla f_\theta(x)$ is the gradient of $f_\theta(x)$ with respect to parameter $\theta$, and $\nabla \rho(\theta)$ is the gradient of the regularization function. We will use the gradient-descent update rule as template for learning max-linear discriminant functions. 
 
\subsubsection{Learning Max-Linear Discriminant Functions}

In general, neither the loss, the regularization function, nor the max-linear discriminant function is differentiable. Consequently, 
the stochastic gradient descent update rule \eqref{eq:sgd} is not applicable for learning max-linear discriminant functions. In this section, we present the main idea to extend update rule \eqref{eq:sgd} to a stochastic subgradient rule for learning max-linear discriminants using a simplified setting. For a more general and technical treatment, we refer to Section \ref{sec:basic-definitions}.

\medskip

We make the following simplifying assumptions:
\begin{enumerate}
\itemsep0em
\item 
We restrict to the two-category case with $\S{Y} = \cbrace{\pm 1}$ as output space. 

\item
The hypothesis space $\S{H}$ is the set $\S{L}_{c, d}$ of max-linear functions with $c$ component functions. 

\item 
The loss function is differentiable. 

\item 
The regularization function is differentiable and of the form $\rho(\theta) = \normS{\theta}{^2}$. 
\end{enumerate}
As advised in the previous sections, we learn a max-linear function for every class separately. Suppose that $f_\theta: \S{X} \rightarrow \S{Y}$ is max-linear with $c$ component functions $f_1, \ldots, f_c: \S{X} \rightarrow \S{Y}$ of the form
\[
f_p(x) = w_p\!{\T}\phi_p(x), 
\] 
where the $w_p \in \R^m$ are the augmented weight vectors. We stack the weight vectors $w_p$ to obtain the parameter vector 
\[
\theta = \begin{pmatrix}
w_1\\
\vdots\\
w_c
\end{pmatrix} \in \R^q
\]
with $q = c\cdot m$. The stochastic subgradient update rule for minimizing the regularized empirical risk $R_N^{\,\rho}(\theta)$ is as follows:
\begin{enumerate}
\item Select an active component $f_p \in \S{A}_f(x)$.
\item Update the parameters according to the rule
\begin{align}
\label{eq:update-rule-max-linear-w-p}
w_p 		&\leftarrow w_p - \eta \args{\ell' \phi_p(x) - \lambda w_p}\\[1ex]
\label{eq:update-rule-max-linear-w-k}
w_k 	&\leftarrow w_k - \eta \lambda w_k
\end{align}
for all $k \in [c]\setminus\cbrace{p}$. 
\end{enumerate}
We make a few comments on the update rules: First, update rule \eqref{eq:update-rule-max-linear-w-p} minimizes the loss and update rule \eqref{eq:update-rule-max-linear-w-k} minimizes the regularization term.  Second, only the weight vector $w_p$ of the selected active component is updated for minimizing the loss but the entire parameter vector $\theta$ summarizing the $c$ weight vectors $w_k$ is updated for minimizing the regularization term. Third, the learning rate $\eta$ absorbs the factor $2$ of the gradient $\nabla \rho(w_k) = 2w_k$. Fourth, it is common practice to exclude the bias from regularization.

In Section \ref{sec:basic-definitions}, we derive a more general subgradient update rule under the assumption that the loss and regularization function are both convex but not necessarily differentiable.

\subsubsection{Examples of Stochastic Subgradient Update Rules}

\begin{table}[t]
\centering
\begin{tabular}{ll@{\qquad}l@{\qquad}l}
\hline
\hline
\\[-2ex]
Adaline 
& $\S{Y} = \cbrace{\pm 1}$ & $\ell = \frac{1}{2}\!\argsS{y - \hat{y}}{^2}$ & $\ell' = -2\args{y - \hat{y}}$\\[1ex]
\hline
\\[-2ex]
Perceptron 
& $\S{Y} = \cbrace{\pm 1}$ & $\ell = \max\cbrace{0, -y\cdot \hat{y}}$ & $\ell' = -y \cdot \mathbb{I}_{\cbrace{\ell > 0}}$\\[1ex]
\hline
\\[-2ex]
Margin Perceptron 
& $\S{Y} = \cbrace{\pm 1}$ & $\ell = \max\cbrace{0, \xi-y\cdot \hat{y}}$ & $\ell' = -y \cdot \mathbb{I}_{\cbrace{\ell > 0}}$\\[1ex]
\hline
\\[-2ex]
Linear SVM
& $\S{Y} = \cbrace{\pm 1}$ & $\ell = \lambda \normS{\theta}{^2} + \max\cbrace{0, 1 - y\cdot \hat{y}}$
&$\ell' = 2\lambda \theta -y \cdot \mathbb{I}_{\cbrace{1-y\cdot \hat{y} > 0}}$\\[1ex]
\hline
\\[-2ex]
Logistic Regression 
& $\S{Y} = \cbrace{0, 1}$ & $\ell = -y \log(\hat{\sigma}) - (1-y) \log(1-\hat{\sigma})$ &$\ell' = \hat{\sigma}-y$ \\
&& $\hat{\sigma} = 1/ \args{1+\exp(-\hat{y})}$\\[1ex]
\hline
\hline
\end{tabular}
\caption{Loss functions $\ell(y, \hat{y})$ as functions of $\hat{y} = f_\theta(x)$ and their derivatives $\ell'$ as a function of $\hat{y}$. As an exception, the derivative $\ell'$ of the linear SVM is given as a function of $\theta$ (see main text for further details). The indicator function $\mathbb{I}_{\text{bool}}$ is $1$ if the expression \texttt{bool} is \texttt{true} and $0$ otherwise.}
\label{tab:loss-functions}
\end{table}

In this section, we present examples of update rule \eqref{eq:update-rule-max-linear-w-p} by specifying the loss function $\ell$, its derivative $\ell'$ and the linear transformations $\phi_p$. Examples of loss functions and their derivatives are summarized in Table \ref{tab:loss-functions}. Unless otherwise stated, we assume that $\ell$ is the perceptron loss and $(x, y) \in \S{D}$ is a training example.

\paragraph*{Perceptron Learning of Max-Linear Functions}
Suppose that the linear transformation $\phi_p = \id$ is the identity. In this case, the hypothesis space consists of max-linear functions 
\[
f(x) = \max \cbrace{f_p(x) \,:\, p \in \S{P}}
\]
with component functions $f_p(x) = w_p\T x$. Update rule \eqref{eq:update-rule-max-linear-w-p} reduces to the standard perceptron learning rule applied on an active component $f_p \in \S{A}_f(x)$: If $f$ misclassifies $x$ then update the weights $w_p$ according to the rule
\begin{align*}
w_p &\leftarrow w_p + \eta \cdot y \cdot x.
\end{align*}
If $f$ correctly classifies $x$ no update rule is applied. Using the notations of Table \ref{tab:loss-functions}, we can compactly rewrite the perceptron update rule as 
\begin{align*}
w_p &\leftarrow w_p + \eta \cdot y \cdot x \cdot \mathbb{I}_{\cbrace{-y f(x) > 0}},
\end{align*}
where the indicator function $\mathbb{I}_{\text{bool}}$ evaluates to $1$ if the expression \texttt{bool} is \texttt{true} and to $0$ otherwise. Thus, the term $\mathbb{I}_{-y f(x) > 0}$ ensures that the update rule is only applied if $f$ misclassifies  $x$.

\paragraph*{Perceptron Learning of Warped-Product Functions}
Suppose that $\S{Q} \subseteq \S{P}_{e,d}$ is a subset. Warped-product functions can be equivalently written as 
\[
f(x) = w \ast x = \max\cbrace{w\T M_p x \,:\, p \in \S{Q}}
\]
where $M_p$ is the warping matrix of warping path $p$. Using $\phi_p = M_p x$ as linear transformation, the perceptron update rule for warped-product functions is of the form
\begin{align*}
w &\leftarrow w + \eta \cdot y\cdot  M_p x \cdot \mathbb{I}_{\cbrace{-y f(x) > 0}};
\end{align*}
where $p$ is an optimal warping path for $w \ast x$ in $\S{Q}$. 

\paragraph*{Perceptron Learning of Elastic-Product Functions}
Suppose that $\S{Q} \subseteq \S{P}_{d,e}$ is a subset. Elastic-product functions can be represented by 
\[
f(x) = W \otimes x = \max\cbrace{\inner{W, \phi_p(x)} \,:\, p \in \S{Q}}
\]
where the linear transformations $\phi_p(x) = X_p$ are the $p$-matrices of $x$ for all $p \in \S{Q}$. The perceptron update rule for elastic-product functions is given by
\begin{align*}
W &\leftarrow W + \eta \cdot y \cdot X_p \cdot \mathbb{I}_{\cbrace{-y f(x) > 0}},
\end{align*}
where $p$ is an optimal warping path for $W \otimes x$ in $\S{Q}$. Observe that only those elements $w_{ij}$ of weight matrix $W$ are updated for which $(i,j)$ is an element of the optimal warping path $p$. 

\paragraph*{Max-Linear Support Vector Machine}

To be consistent with the stochastic subgradient update rules \eqref{eq:update-rule-max-linear-w-p} and \eqref{eq:update-rule-max-linear-w-k}, we treat the linear SVM as a special case of a regularized margin perceptron. Suppose that $\theta \in \R^q$ is a parameter vector. Then the loss of the linear SVM is of the form
\[
\ell(y, \hat{y}) = \max\cbrace{0, 1 - y\cdot \hat{y}} + \lambda \normS{\theta}{^2},
\]
where $\hat{y} = f_\theta(x)$. The first term corresponds to the loss of the margin perceptron with parameter $\xi = 1$. The second term is regularization term. The derivative of $\ell(y, \hat{y})$ as a function of $\hat{y}$ is defined by
\[
\ell' = -y \cdot \mathbb{I}_{\cbrace{1-y\cdot \hat{y} > 0}}. 
\]
and the gradient of the regularization terms is $2 \lambda \theta $. Combining both terms gives the derivative of the loss as a function of the parameter vector $\theta$ as shown in Table \ref{tab:loss-functions}.

\medskip

We extend the update rule of the linear SVM to max-linear discriminant functions $f_\theta$ with parameter vector $\theta \in \R^q$ and component functions $f_p(x) = w_p\T x$. We select an active component $f_p$ of $f_\theta$ at $x$. Then the update rule of the max-linear SVM is given by
\begin{align*}
w_p 	&\leftarrow w_p - \eta \cdot \args{\lambda \cdot w_p - y \cdot x \cdot \mathbb{I}_{\cbrace{1-y\cdot f_\theta(x) > 0}}}\\
w_k	&\leftarrow w_k - \eta \lambda w_k
\end{align*}
for all $k \in \S{P} \setminus \cbrace{p}$. We obtain the warped-product and elastic-product SVM by substituting $M_p x$ and $X_p$, rep., for $x$ in the first update rule. 

\section{Experiments}

\newcommand{\asm}{\textsc{sm}}
\newcommand{\awp}{\textsc{wp}}
\newcommand{\aep}{\textsc{ep}}
\newcommand{\aml}{\textsc{ml}}
\newcommand{\ann}{\textsc{nn}}
\newcommand{\alvq}{\textsc{glvq}}

The goal of this section is to assess the performance and behavior of warped-linear classifiers. Three experiments were conducted:
\begin{enumerate}
\itemsep0em
\item Comparison of classifiers in DTW spaces
\item Comparison of polyhedral classifiers
\item Experiments on label dependency
\end{enumerate}

\subsection{Comparison of Classifiers in DTW Spaces}\label{subsec:exp01}

This section compares the performance of elastic-product classifiers against state-of-the-art prototype-based classifiers using the DTW-distance.

\subsubsection{Data}

We used $29$ datasets from the UCR time series classification and clustering repository \cite{Chen2015}. The datasets were chosen to cover various characteristics such as application domain, length of time series, number of classes, and sample size. Table \ref{tab:results} shows the datasets. 

\subsubsection{Algorithms}

The following classifiers were considered:
\begin{center}
\begin{tabular}{l@{\qquad}l@{\qquad}c}
\hline
Notation & Algorithm & Reference\\
\hline
\\[-2ex]
\ann   & nearest neighbor classifier	  & \cite{Hastie2001}\\
\alvq  & asymmetric generalized LVQ 	  & \cite{Jain2018}\\
\asm   & softmax regression            & \cite{Hastie2001}\\
\aep   & elastic-product classifier    & proposed\\
\hline
\end{tabular}
\end{center}
The \asm ~and \aep ~methods used the multinomial logistic loss. No restrictions were imposed on the set of warping paths for \aep. 
We compared \aep ~with \ann ~and \alvq, because both methods are nearest neighbor methods that also directly operate on DTW-spaces. We compared \aep ~with \asm ~to assess the importance of incorporating time-warp invariance. Recall that \asm ~can be regarded as a special case of \aep ~with elasticity one. Note that \asm ~can be applied, because all time series of a UCR dataset have identical length.

\subsubsection{Experimental Protocol}

For every dataset, we conducted ten-fold cross validation and reported the average classification accuracy, briefly called accuracy, henceforth. Since the experimental protocols and the folds of the datasets are identical, the results of \ann ~and \alvq ~were taken from \cite{Jain2018}.

The \asm ~and \aep ~classifiers minimized the empirical risk by applying the stochastic subgradient method with \emph{Adaptive Moment Estimation} (ADAM) proposed by \cite{Kingma2015}. The decay rates of ADAM for the first and second moment were set to $\beta_1 = 0.9$ and $\beta_2 = 0.999$, respectively. The maximum number of epochs (cycles through the training set) were set to $5,000$ and the maximum number of consecutive epochs without improvement to $100$, where improvement refers to a decrease of the empirical risk. The initial learning rates of all four classifiers were selected according to the following procedure: 

\begin{small}
\begin{algorithm}{Selection of initial learning rate}
\label{alg:initial_lr}
\atab{1} \textbf{Procedure}:\\
\atab{2} initialize learning rate $\eta = 0.8$\\
\atab{2} initialize classifier\\
\atab{2} \textbf{repeat} \\
\atab{4} set $\eta = \eta / 2$ and $s = 0$ \hfill //$s$ counts the number of epochs without improvement)\\
\atab{4} \textbf{for} every epoch $t \leq 100$ \textbf{do}\\
\atab{6} apply stochastic learning rule\\
\atab{6} set $s = s + 1$ if empirical risk is not lower\\
\atab{6} compute ratio $\rho = s/t$\\
\atab{2}	\textbf{until} $\rho < 0.2$\\
\atab{1} \textbf{Return:} learning rate $\eta$ 
\end{algorithm}
\end{small}
For \aep, we selected the  elasticity $e \in \cbrace{1, 2, 3, 4, 5, 7, 10, 15, 20, 25, 30, 35, 40, 45, 50}$ that gave the minimum empirical risk. 

\subsubsection{Results}

Table \ref{tab:ranks} shows the rank distributions and Figure \ref{fig:pairwise} shows the pairwise comparison of the four classifiers based on the results presented in Table \ref{tab:results}. 

\begin{table}[b]
\centering
\begin{tabular}{ll@{\qquad}rrrr@{\qquad}cc}
\hline
\hline
& &\multicolumn{4}{c}{Rank} & & \\
\multicolumn{2}{l}{Classifier} & 1 & 2 & 3 & 4 & avg & std\\
\hline
\\[-2ex]
\ann    &nearest-neighbor& 10 & 7 & 6 & 6 & 2.3 & 1.1\\
\alvq   &generalized LVQ& 9 & 8 & 7 & 5 & 2.3 & 1.1\\
\asm    &softmax regression& 2 & 4 & 5 & 18 & 3.3 & 1.0\\
\aep    &elastic-product classifier& 11 & 7 & 11 & 0 & 2.0 & 0.9\\
\hline
\hline
\end{tabular}
\caption{Rank distribution, average rank, and standard deviation of the four classifiers nn, glvq, sm, and ep based on the results shown in Table \ref{tab:results}. The average accuracy of every classifier on a given dataset was ranked, where ranks go from $1$ (highest accuracy) to $4$ (lowest accuracy).}
\label{tab:ranks}
\end{table}

\begin{figure}[t]
\centering
\begin{tabular}{c@{\hspace{3cm}}c}
\textbf{Winning Percentage} & \textbf{Mean Percentage Difference}\\[1ex]
\includegraphics[width=0.35\textwidth]{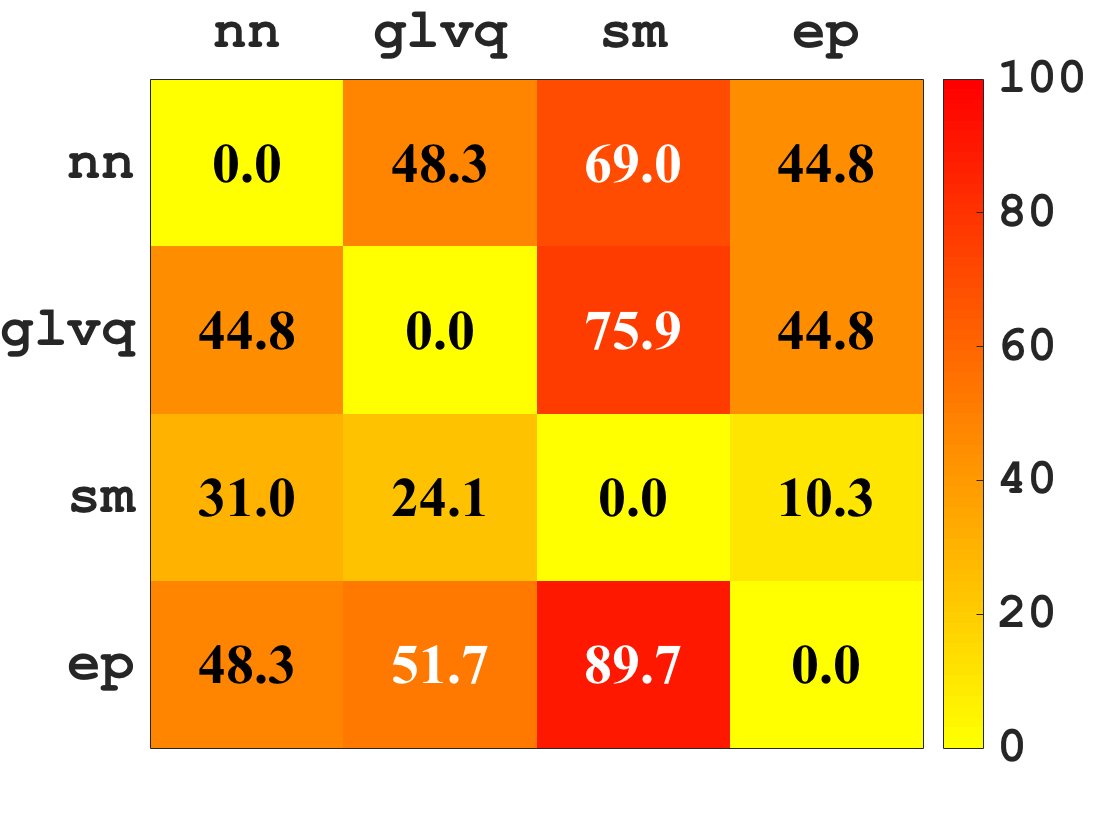} &
\includegraphics[width=0.35\textwidth]{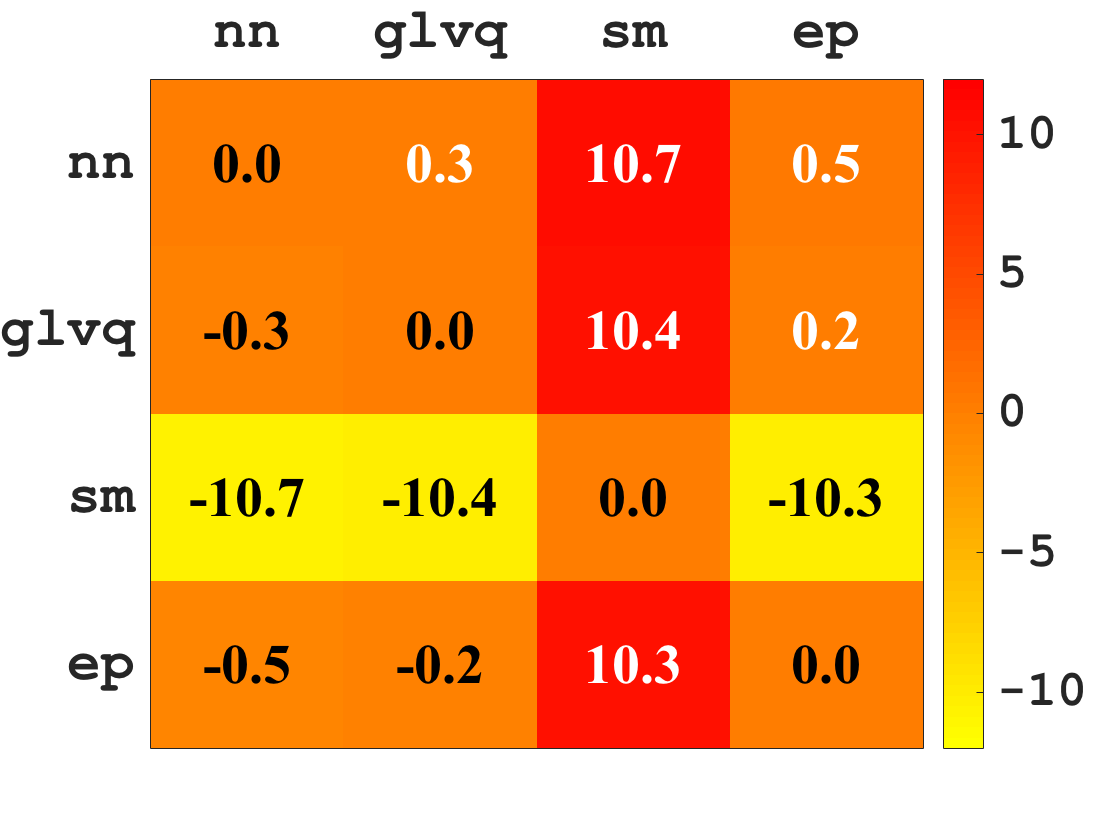}
\end{tabular}
\caption{Pairwise comparison of seven classifiers based on the results  shown in Table \ref{tab:results}. Left: Pairwise winning percentages $w_{ij}$, where classifier in row $i$ wins $w_{ij}$ percentages of all competitions against the classifier in column $j$. Right: Pairwise mean percentage difference $a_{ij}$ in accuracy, where the accuracy of classifier in row $i$ is $a_{ij}$ percentages better on average than the accuracy of the classifier in column $j$. A definition of both measures is given in Appendix \ref{sec:performance-measures}.} 
\label{fig:pairwise}
\end{figure}

The linear classifier \asm ~performed worst by a large margin in an overall comparison (see Table \ref{tab:ranks}) and in a pairwise comparison (see Figure \ref{fig:pairwise}). Since \alvq ~with a single prototype per class performed significantly better, these findings suggest that \asm ~is unable to filter out the variation in temporal dynamics. This conclusion is in line with the claim that the geometric structure of time series data is non-Euclidean and therefore linear models are often not an appropriate choice.

In an overall comparison, \aep ~performed best with average rank $2.0$ followed by the other two DTW classifiers, \ann ~and \alvq, both with average rank $2.3$. In a pairwise comparison, the three DTW classifiers are comparable with slight advantages for \aep ~with respect to winning percentages and \ann ~with respect to mean percentage difference. The rank distribution suggests that the three DTW classifiers complement one another with regard to predictive performance.

The advantage of \aep ~over the other two DTW classifiers \ann ~and \alvq ~is its efficiency. To see this, Table \ref{tab:time} shows the computational cost for classifying a single test example under the assumption that all time series are of length $l$. On average, \aep ~is almost three orders of magnitudes faster than \ann ~and about one order of magnitude faster than \alvq. The speed-up factor of \aep ~compared to prototpye-based methods with a single prototype per class is $l/e$, where the length $l$ ranges from $24$ to $720$ and the elasticity $e$ is selected from $1$ to $50$ in this experiment. Note that elasticity $e = 1$ refers to the linear model \asm. Thus, the elasticity parameter $e$ can be regarded as a control parameter to trade computation time and complexity of the function space (VC dimension).

\begin{table}[h]
\centering
\begin{tabular}{l@{\qquad}c@{\qquad}c@{\qquad}c@{\qquad}c}
\hline
\hline
classifier & \ann & \alvq & \asm & \aep \\
\hline
\\[-2ex]
complexity & $\mathcal{O}\args{N\cdot l^2}$ & $\mathcal{O}\args{K\cdot l^2}$ & $\mathcal{O}\args{K\cdot l}$ & $\mathcal{O}\args{K\cdot l\cdot e}$\\
factor & $1$ & $N/K$ & $l \cdot N / K$  & $(l \cdot N) / (e \cdot K)$ \\
avg factor& $1$ & $123.3$ & $21112.7$ & $844.5$\\
\hline
\hline
\multicolumn{5}{l}{\footnotesize
$N$ = \# training examples $\bullet$ $K$ = \# classes $\bullet$ $l$ = length of time series $\bullet$ $e$ = elasticity.}
\end{tabular}
\caption{Computational effort for classifying a single test example. The first row shows the complexity under the assumption that all time series are of fixed length $l$. The second row shows the speed-up factor of a classifier compared to the nn classifier. The last row shows the average speed-up factor over all UCR datasets from Table \ref{tab:results}.}
\label{tab:time}
\end{table}

\subsection{Comparison of Polyhedral Classifiers}\label{subsec:exp02}

This section compares the performance of different polyhedral classifiers. 

\subsubsection{Data}

We used $15$ datasets from the UCR time series classification and clustering repository \cite{Chen2015} and $12$ datasets from the UCI Machine Learning repository \cite{Lichman2013}. Table \ref{tab:results:ucruci} shows the selected UCR and UCI datasets.

\subsubsection{Algorithms}
We considered the following classifiers:
\begin{itemize}
\itemsep0em 
\item Softmax regression (\asm)
\item Warped-product classifier (\awp)
\item Elastic-product classifier (\aep)
\item Max-linear classifier (\aml)
\end{itemize}
All classifiers used the multinomial logistic loss. We imposed no restrictions on the set of warping paths for both warped-linear classifiers \awp ~and \aep. Warped-product classifiers were applied on an augmented input space with leading and trailing zero as suggested by Prop.~\ref{prop:W'_d = L_d}.

\subsubsection{Experimental Protocol}

We performed holdout validation on all datasets. For time series data, we used the train-test split provided by the UCR repository. We randomly split the datasets from the UCI repository into a training and test set with a ratio of $2:1$. 

All classifiers applied the stochastic subgradient method using ADAM with decay rates $\beta_1 = 0.9$ and $\beta_2 = 0.999$ for the first and second moment, respectively. The maximum number of epochs were set to $5,000$ and the maximum number of consecutive epochs without improvement to $250$. The initial learning rates of of all four classifiers were picked according to Algorithm \ref{alg:initial_lr}. The elasticity $e \in \cbrace{1, 2, 3, 4, 5, 7, 10, 15, 20}$ of \awp, \aep, and \aml ~with minimum empirical risk was selected.

\subsubsection{Results and Discussion}

\begin{table}[t]
\centering
\begin{tabular}{l@{\qquad}rrrrrr@{\qquad\qquad}l@{\qquad}rrrrrr}
\hline
\hline
& \multicolumn{4}{c}{rank} & & & & \multicolumn{4}{c}{rank} \\
UCR & 1 & 2 & 3 & 4 & avg & std & UCI &  1 & 2 & 3 & 4 & avg & std\\
\hline
\\[-2ex]
\asm & 1 & 2 & 6 & 6 & 3.1 & 0.88 & \asm & 5 & 1 & 5 & 1 & 1.7 & 1.07\\
\awp & 8 & 2 & 1 & 4 & 2.1 & 1.29 & \awp & 0 & 0 & 2 & 10 & 3.1 & 0.37\\
\aep & 6 & 9 & 0 & 0 & 1.6 & 0.49 & \aep & 5 & 5 & 1 & 1 & 1.5 & 0.90\\
\aml & 3 & 3 & 6 & 3 & 2.6 & 1.02 & \aml & 4 & 4 & 4 & 0 & 1.6 & 0.82\\
\hline
\hline
\end{tabular}
\caption{Rank distribution, average rank, and standard deviation of the four classifiers \asm, \awp, \aep, and \aml ~on selected UCR and UCI datasets based on the results shown in Table \ref{tab:results:ucruci}. The average accuracy of every classifier on a given dataset was ranked, where ranks go from $1$ (highest accuracy) to $4$ (lowest accuracy).}
\label{tab:ranks:ucruci}
\end{table}

\begin{figure}[t]
\centering
\begin{tabular}{c@{\qquad\qquad}c}
UCR & UCI\\[1ex]
\includegraphics[width=0.35\textwidth]{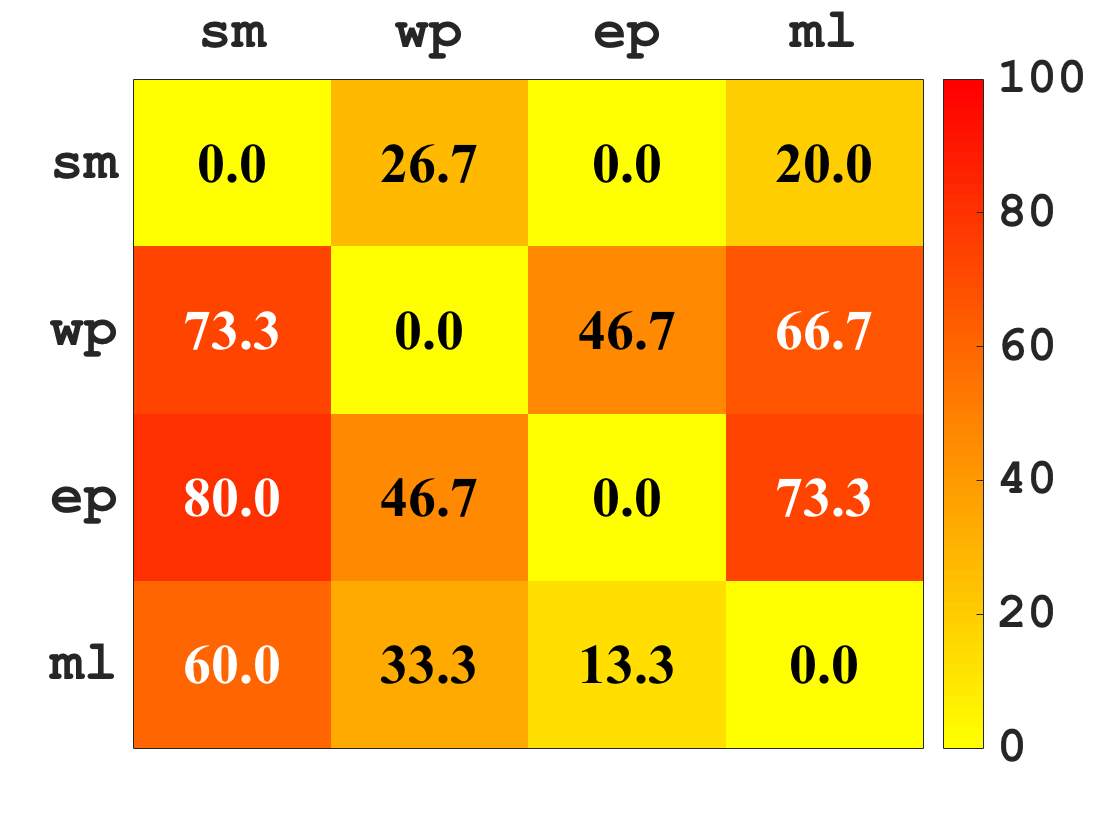} &
\includegraphics[width=0.35\textwidth]{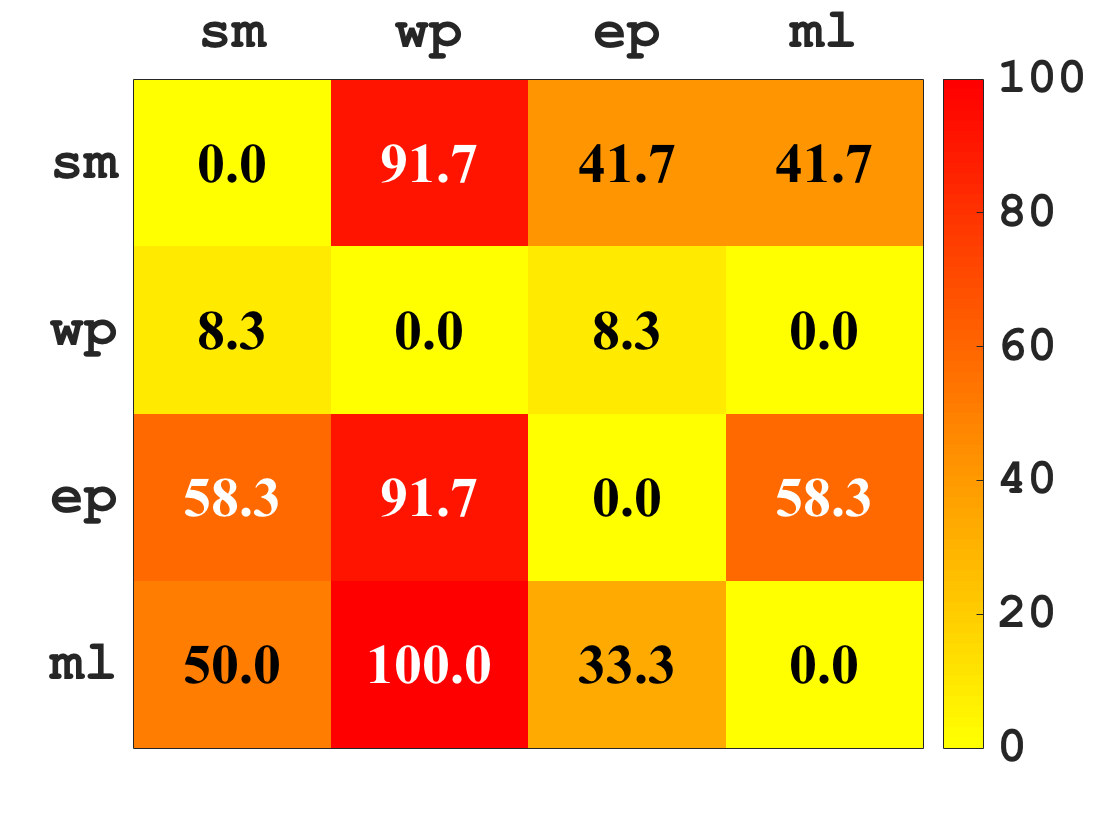}
\end{tabular}
\caption{Pairwise comparison of the four classifiers sm, wp, ep, and ml based on the results of Table \ref{tab:results:ucruci}. The winning percentage $w_{ij}$ shows that classifier in row $i$ wins $w_{ij}$ percentages of all competitions against the classifier in column $j$.} 
\label{fig:ucruci_pw}
\end{figure}

Table \ref{tab:ranks:ucruci} shows the rank distributions and Figure \ref{fig:ucruci_pw} shows the pairwise comparison of the four polyhedral classifiers based on the results presented in Table \ref{tab:results:ucruci}.

The linear classifier \asm ~is not competitive on the UCR time series data but performed only slightly worse than the best polyhedral classifiers on the UCI vector datasets (see Table \ref{tab:ranks:ucruci} and Figure \ref{fig:ucruci_pw}). These findings are in line with the observation of the first experiment and confirm that \asm ~fails to capture the variations in temporal dynamics. 

Although  \awp, \aep, and \aml ~essentially represent the same class of functions, their performances substantially differ. On the UCR time series datasets \aep ~performed best followed by \awp ~and \aml. On the UCI datasets, \aep ~and \aml ~performed best on comparable level, while \awp ~was ranked last by a large margin. These findings suggest that the warped-liner classifiers are better suited for time series classification than the max-linear classifier. One possible explanation for this phenomenon could be that max-linear classifiers only update a single hyperplane corresponding to the active component, whereas updating an active hyperplane of warped-product classifiers simultaneously updates weights of non-active hyperplanes by construction. This could be possibly advantageous for learning time-warp invariance. 

An explanation why the elastic-product classifier \aep ~performed comparable with the max-linear classifier \aml ~on UCI datasets, whereas the warped-product classifier \awp ~failed miserably could be as follows: By construction, \aep ~is sufficiently flexible to learn different hyperplanes that share only few weights, whereas massive weight sharing of \awp ~may result in poor classification performance of vector data. 

Finally, it should be noted that \awp ~is computationally demanding, because the length of the weight sequence is multiples longer than the length of the time series to be classified. Suppose that all time series are of length $l$ and let $e$ be the elasticity. Then the complexity of classifying a single test instance by \awp ~is $\mathcal{O}(el^2)$, whereas the complexity of the same task is $\mathcal{O}(el)$ for \aep ~is $\mathcal{O}(el)$. Overall, these and the above findings suggest to prefer \aep ~over \awp ~in time series classification.

\subsection{Label Dependency}

This section studies the problem of label dependency for elastic-product classifiers.

\subsubsection{Illustrative Examples}

\begin{figure}[t]
\centering
 \includegraphics[width=0.49\textwidth]{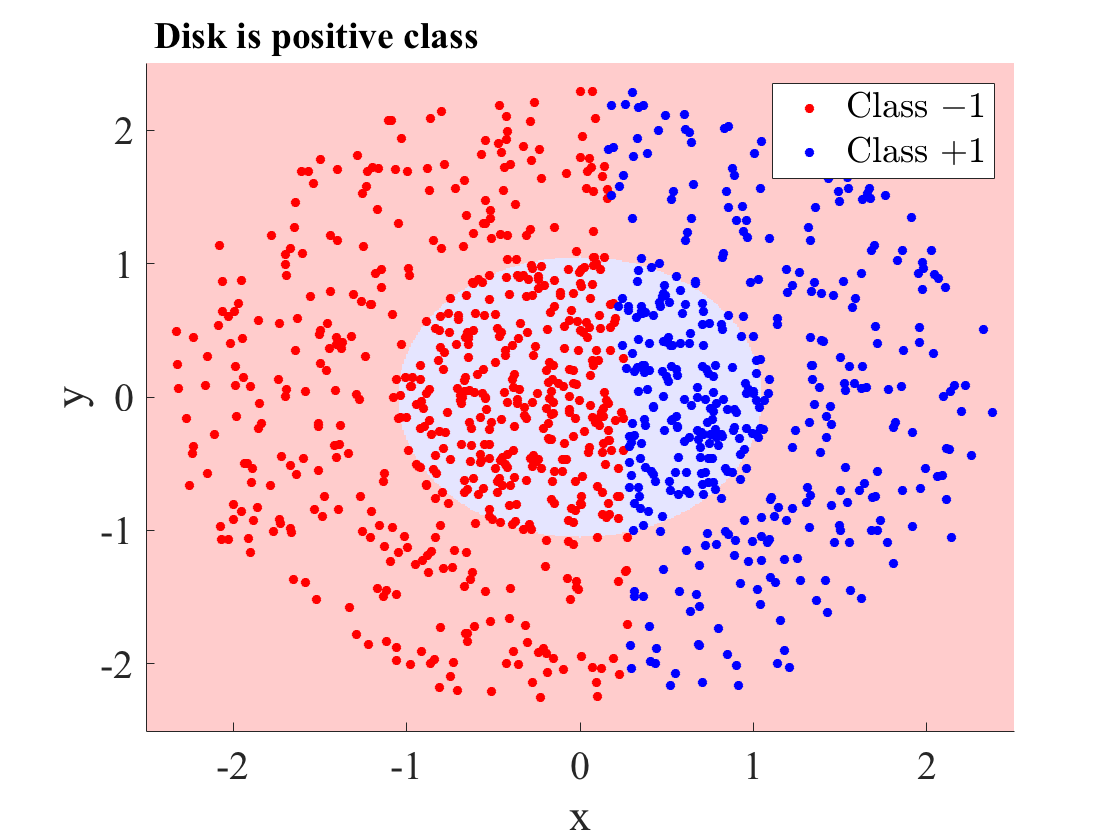}
 \hfill
 \includegraphics[width=0.49\textwidth]{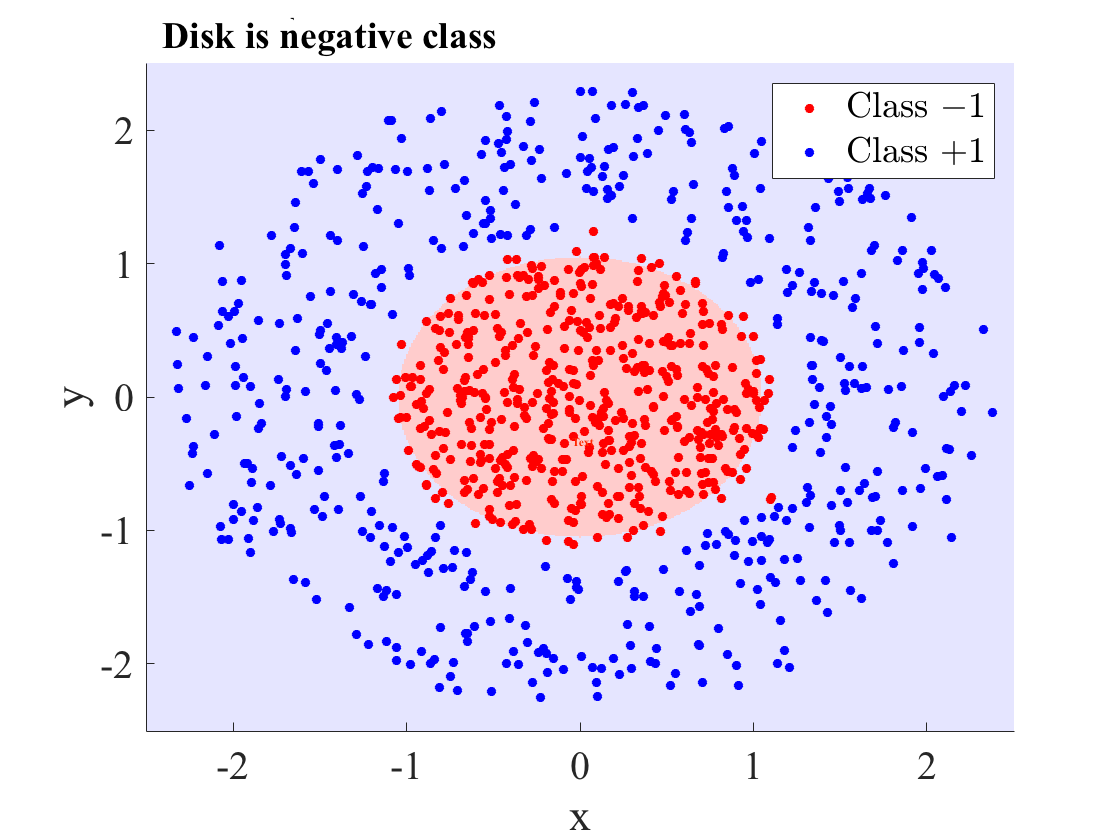}
\caption{Effect on choice of labeling function on classification error. The true class regions are the unit disk $\S{D}$ and its complement $\overline{\S{D}}$. A true class region is shaded red if negative and blue if positive. The left (right) plot assumes that the disk $\S{D}$ is the positive (negative) class region. The colored dots show the classification results of an elastic-product classifier $f$ with elasticity $e = 10$. A red (blue) dot represents a point assigned to the negative (positive) class by $f$.}
\label{fig:ex_circle} 
\end{figure}

\noindent
\emph{Two-Category Problem.\ }
To illustrate label dependency of elastic-product classifiers using a single discriminant function for two-category problems, we consider the problem of separating points inside a unit disk $\S{D}$ from points outside the disk $\overline{\S{D}}$. We randomly sampled $500$ points from each class region. Then we conducted two experiments: In the first (second) experiment, we regarded the disk as positive (negative) class region. In both experiments, we applied an elastic-product classifier with a single discriminant of elasticity $e = 10$. 

Figure \ref{fig:ex_circle} shows the results of both experiments. The plots of Figure \ref{fig:ex_circle} indicate that an elastic-product classifier with single discriminant function succeeds (fails) to broadly separate both classes if the negative class region is convex (non-convex). This result confirms the theoretical findings in Section \ref{subsec:separability}. As a solution to label dependency, Section \ref{subsec:separability} proposes to use $K$ discriminant functions if there are $K$ classes. For the disk classification problem, the results obtained by the elastic-product classifier with two discriminant functions are similar to the right plot of Figure \ref{fig:ex_circle} irrespective of how both classes are labeled. 

\medskip

\begin{figure}[h]
\centering
\includegraphics[width=0.49\textwidth]{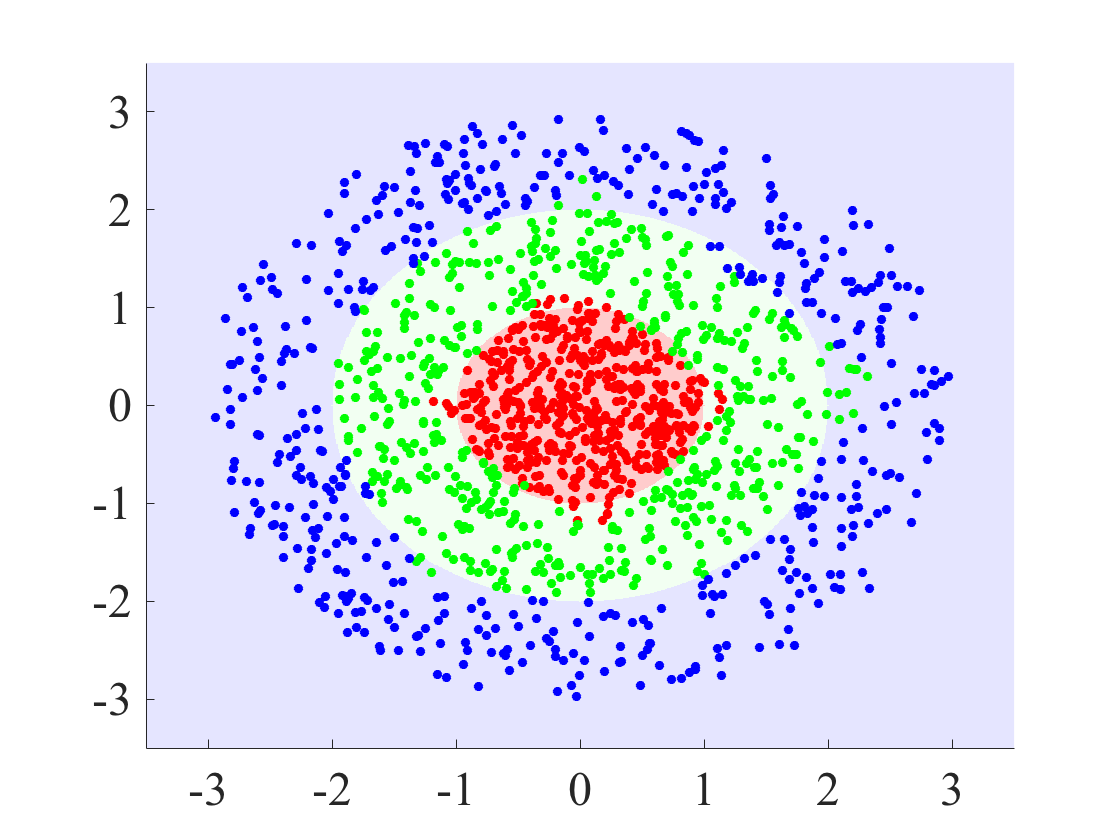}
\hfill
\includegraphics[width=0.49\textwidth]{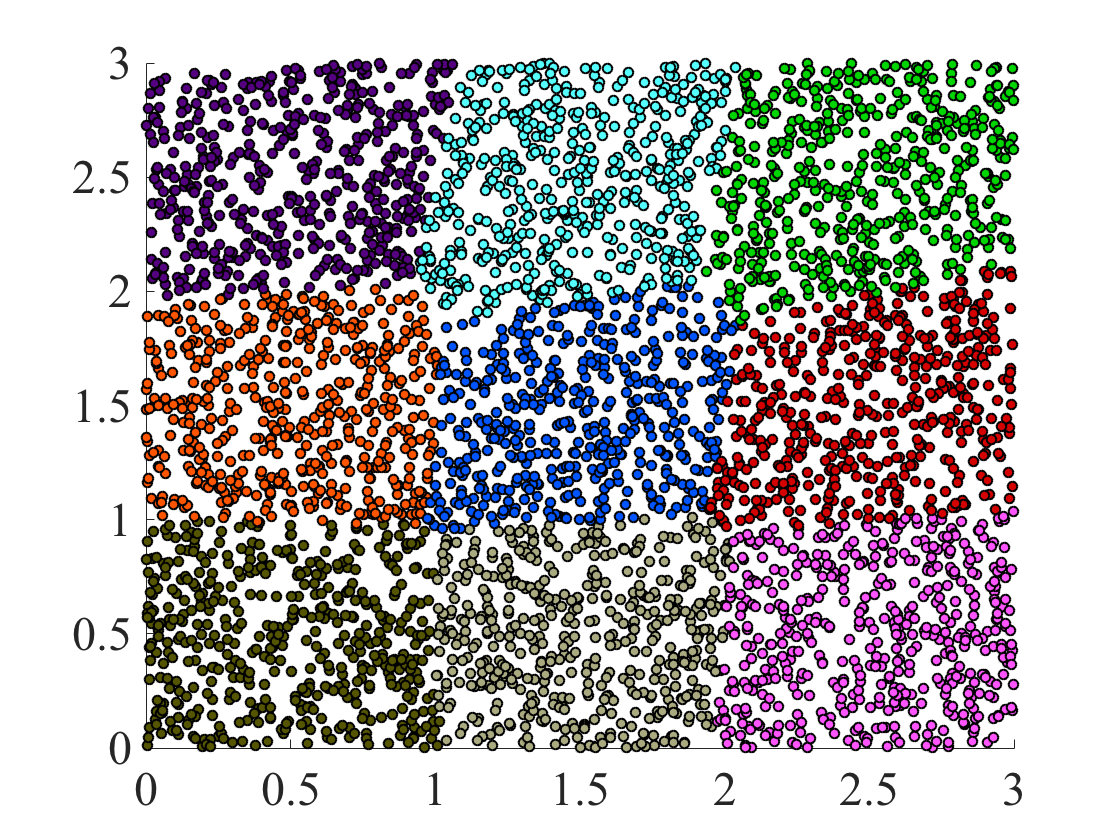}
\caption{Results of elastic-product classifiers with $K$ discriminant functions on two $K$-category problems. Left: The three class regions are a unit disk with center $(0,0)$ and radius $1$ (shaded red), a ring around the disk (shaded green), and the complement of both other class regions (shaded blue). Right: The nine class regions are unit squares arranged in a $3\times 3$ grid. Both: Data points are shown by dots. The color of a dot refers to the class label predicted by the respective elastic-product classifiers.}
\label{fig:ex_rings} 
\end{figure}

\noindent
\emph{Multi-Category Problems.\ }
We tested the elastic-product classifier with $K$ discriminant functions on a $3$- and on a $9$-category problem with multiple convex sets (see Figure \ref{fig:ex_rings}). As in the disk classification problem, we sampled $500$ points per class and applied the elastic-product classifier with elasticity $e = 10$. Figure \ref{fig:ex_rings} shows the results of both experiments. Both plots show that the elastic-product classifier broadly separates the multiple convex class regions.

\subsubsection{UCR Time Series}

In this experiment, we empirically studied label dependency using $15$ two-category problems from the UCR time series classification and clustering repository \cite{Chen2015}. We applied the following variants of elastic-product classifiers:
\begin{enumerate} 
\itemsep0em
\item \aep$_{\min}$: one discriminant function and `unfavourable' labeling of the training examples. 
\item \aep$_{\max}$: one discriminant function and `favourable' labeling of the training examples.
\item \aep$_{2\phantom{ax.}}$: two discriminant functions and randomly chosen labeling function.
\end{enumerate}
By favourable (unfavourable) labeling we mean the labeling of the training examples that resulted in a higher (lower) classification accuracy on the training set. All variants of \aep ~used elasticity $e = 5$. The initial learning rate was selected as in Algorithm \ref{alg:initial_lr}. The experimental protocol was holdout validation using the train-test set split provided by the UCR repository.

Table \ref{tab:ucr-label-dependency} shows the classification accuracies of the three classifiers. The overall results show that there are slight advances for \aep$_{\max}$ over \aep$_{\max}$. This finding indicates that using two discriminant solves the label dependency problem. 

\begin{table}[t]
\centering
\begin{tabular}{l@{\qquad}rrr@{\quad}c}
\hline
\hline
UCR dataset & \aep$_{\min}$ & \aep$_{\max}$ & \aep$_2\phantom{a}$ & \aep$_2$ $-$ \aep$_{\max}$\\
\hline
\\[-2ex]
BirdChicken  & 60.00  & 80.00  & 80.00 & $\sim$ \\
Coffee  & 89.29  & 96.43  & 100.00 & $+$ \\
DistalPhalanxOutlineAgeGroup  & 68.00  & 82.00  & 82.00 & $\sim$  \\
FordA  & 59.34  & 61.79  & 71.84 & $+$ \\
GunPoint  & 82.67  & 90.67  & 85.33 & $-$  \\
Ham  & 69.52  & 70.48  & 61.90  & $-$ \\
HandOutlines  & 83.90  & 85.70  & 86.50 & $\sim$  \\
ItalyPowerDemand  & 96.40  & 97.28  & 96.89 & $\sim$ \\
Lighting2  & 62.30  & 62.30  & 72.13 & $+$ \\
MoteStrain  & 82.59  & 86.42  & 84.42 & $-$ \\
ProximalPhalanxOutlineCorrect  & 85.57  & 86.94  & 85.22 & $-$ \\
SonyAIBORobotSurface  & 73.88  & 83.69  & 85.36  & $+$\\
TwoLeadECG  & 77.17  & 82.79  & 84.99  & $+$\\
Wafer  & 97.11  & 97.73  & 99.37  & $+$\\
Yoga  & 65.93  & 70.70  & 80.27  & $+$\\
\hline
\hline
\end{tabular}
\caption{Test classification accuracies in percentage of \aep$_{\min}$, \aep$_{\max}$, and \aep$_2$.The last column signifies whether \aep$_2$ performed better ($+$), comparable ($\sim$), or worse ($-$) than \aep$_{\max}$.}
\label{tab:ucr-label-dependency}
\end{table}

\section{Conclusion}
Warped-linear models are time-warp invariant analogues of linear models. Under mild assumptions, they are equivalent to polyhedral classifiers. This equivalence relationship is useful, because its simplifies analysis of warped-linear functions by reducing to max-linear functions. Both, analysis of the label dependency problem and derivation of the stochastic subgradient method exploited the proposed equivalence. Empirical results in time series classification suggest that elastic-product classifiers are an efficient and complementary alternative to nearest neighbor and prototype-based methods in DTW spaces. Inspired by linear models in Euclidean spaces, future work aims at analysis of warped-linear functions in DTW spaces and construction of advanced classifiers such as piecewise warped-linear classifiers and warped neural networks.

\clearpage

\begin{appendix}
\section{Results}

\subsection{Table \ref{tab:results}: Classification Accuracy on UCR Data (Section \ref{subsec:exp01})}

\begin{table}[ht]
\centering
\begin{tabular}{lr@{\qquad}rrrr}
\hline
\hline
UCR dataset & nn\; & glvq & sm & ep\; \\
\hline
\\[-2ex]
Beef & 53.3 & 63.3 & 84.0 & 81.0 \\
CBF & 99.9 & 99.6 & 97.9 & 98.9 \\
ChlorineConcentration & 99.6 & 71.6 & 86.9 & 99.8 \\
Coffee & 100.0 & 98.2 & 96.7 & 100.0 \\
ECG200 & 83.5 & 80.5 & 86.0 & 88.4 \\
ECG5000 & 93.3 & 94.3 & 94.1 & 93.3 \\
ECGFiveDays & 99.2 & 99.1 & 99.6 & 99.7 \\
ElectricDevices & 79.2 & 78.2 & 53.9 & 74.7 \\
FaceFour & 92.9 & 92.0 & 91.7 & 92.3 \\
FacesUCR & 97.8 & 97.2 & 86.8 & 94.0 \\
FISH & 80.3 & 90.9 & 86.3 & 90.0 \\
GunPoint & 91.5 & 97.0 & 87.0 & 96.0 \\
Ham	& 72.4 & 74.8 & 81.8	& 84.9 \\
ItalyPowerDemand & 95.8 & 95.3 & 97.0 & 95.9 \\
Lighting2 & 89.3 & 76.0 & 61.6 & 65.8 \\
Lighting7 & 71.3 & 82.5 & 56.9 & 63.1 \\
MedicalImages & 80.7 & 71.7 & 63.9 & 73.2 \\
OliveOil & 85.0 & 85.0 & 52.7 & 81.0 \\
ProximalPhalanxOutlineAgeGroup & 75.5 & 83.6 & 82.2 & 83.6 \\
ProximalPhalanxOutlineCorrect & 82.0 & 85.1 & 78.9 & 86.3 \\
ProximalPhalanxTW & 77.5 & 81.2 & 81.0 & 82.2 \\
RefrigerationDevices & 60.7 & 62.9 & 37.5 & 47.9 \\
Strawberry & 96.5 & 94.1 & 96.1 & 97.9 \\
SwedishLeaf & 82.0 & 87.6 & 80.7 & 87.7 \\
synthetic control & 99.2 & 99.5 & 83.0 & 94.2 \\
ToeSegmentation1 & 85.8 & 92.5 & 52.6 & 73.8 \\
Trace & 100.0 & 100.0 & 70.0 & 89.0 \\
wafer & 99.4 & 96.5 & 94.0 & 99.8 \\
yoga & 93.9 & 73.2 & 69.5 & 93.2\\ 
\hline
\hline
\end{tabular}
\caption{Average accuracy in percentage of the four classifiers nn, glvq, sm, ep on $29$ UCR datasets obtained from $10$-fold cross validation.}
\label{tab:results}
\end{table}

\newpage 

\subsection{Table \ref{tab:results:ucruci}: Classification Accuracy on UCR and UCI Data (Section \ref{subsec:exp02})}

\begin{table}[h]
\centering
\begin{tabular}{l@{\quad}rrrr}
\hline
\hline
UCR dataset & sm & wp & ep & ml \\
\hline
\\[-2ex]
DistalPhalanxOutlineAgeGroup & 76.00 & 82.50 & 85.00 & 79.25 \\
ECG5000 & 92.27 & 93.29 & 93.40 & 92.69 \\
ElectricDevices & 45.92 & 68.67 & 56.05 & 42.17 \\
Gun Point & 80.67 & 96.67 & 85.33 & 84.00 \\
ItalyPowerDemand & 96.89 & 90.18 & 97.08 & 97.18 \\
MedicalImages & 55.39 & 66.84 & 64.21 & 63.82 \\
MiddlePhalanxTW & 61.40 & 59.90 & 61.40 & 61.65 \\
Plane & 95.24 & 96.19 & 96.19 & 95.24 \\
ProximalPhalanxOutlineAgeGroup & 82.44 & 83.90 & 85.37 & 84.88 \\
ProximalPhalanxOutlineCorrect & 84.19 & 82.13 & 88.32 & 86.94 \\
ProximalPhalanxTW & 78.75 & 80.50 & 78.75 & 78.75 \\
SonyAIBORobotSurface & 76.54 & 87.85 & 83.53 & 70.05 \\
SwedishLeaf & 79.20 & 86.08 & 81.12 & 78.56 \\
synthetic control & 80.00 & 98.33 & 92.33 & 86.33 \\
TwoLeadECG & 93.94 & 80.60 & 93.94 & 93.94 \\
\hline
\hline
\\
\hline
\hline
UCI dataset & sm & wp & ep & ml \\
\hline
\\[-2ex]
balance & 90.87 & 80.77 & 92.31 & 91.83 \\
banknote & 98.69 & 96.94 & 100.00 & 100.00 \\
ecoli & 81.58 & 67.54 & 75.44 & 73.68 \\
eye & 60.42 & 68.54 & 82.36 & 88.68 \\
glass & 61.97 & 59.15 & 63.38 & 64.79 \\
ionosphere & 87.18 & 84.62 & 89.74 & 88.03 \\
iris & 94.12 & 88.24 & 92.16 & 94.12 \\
occupancy & 99.01 & 97.26 & 98.96 & 98.85 \\
pima & 77.73 & 67.58 & 67.19 & 71.48 \\
sonar & 71.01 & 68.12 & 88.41 & 84.06 \\
whitewine & 53.52 & 49.54 & 53.83 & 53.34 \\
yeast & 58.59 & 46.87 & 57.37 & 54.55 \\
\hline
\hline
\end{tabular}
\caption{Classification accuracy in percentage of wp, ep, and ml on selected UCR and UCI datasets.}
\label{tab:results:ucruci}
\end{table}

\clearpage
\section{Performance Measures}\label{sec:performance-measures}

This section describes the pairwise winning percentage and pairwise mean percentage difference. 

\subsection{Winning Percentage}

The pairwise winning percentages are summarized in a matrix $W = (w_{ij})$. The winning percentage $w_{ij}$ is the fraction of datasets for which the accuracy of the classifier in row $i$ is strictly higher than the accuracy of the classifier in column $j$. Formally, the winning percentage $w_{ij}$ is defined by
\[
w_{ij} = 100 \cdot \frac{\abs{\cbrace{d \in \S{D} \,:\, \text{acc}_d(j) < \text{acc}_d(i)}}}{\abs{\S{D}}}
\]
where $\text{acc}_d(i)$ is the accuracy of the classifier in row $i$ on dataset $d$, and $\text{acc}_d(j)$ is the accuracy of the classifier in column $j$ on $d$. The percentage $w_{ij}^{eq}$ of ties between classifiers $i$ and $j$ can be inferred by 
\[
w_{ij}^{eq} = 100-w_{ij}-w_{ji}.
\]

\subsection{Pairwise Mean Percentage Difference}
The pairwise mean percentage differences are summarized in a matrix $A = (a_{ij})$. The mean percentage difference $a_{ij}$ between the classifier in row $i$ and the classifier in column $j$ is defined by
\[
a_{ij} = 100 \cdot\frac{2}{\abs{\S{D}}}\sum_{d \in \S{D}} \cdot \frac{\text{acc}_d(i)-\text{acc}_d(j)}{\text{acc}_d(i) + \text{acc}_d(j)},
\]
Positive (negative) values $a_{ij}$ mean that the average accuracy of the row classifier was higher (lower) on average than the average accuracy of the column classifier.

\section{Proofs}

This section presents the proofs of the theoretical results.

\subsection{Notations}
Throughout this section, we use the following notations:

\bigskip

\begin{tabular}{l@{\quad:\quad}l}
\multicolumn{2}{l}{\textbf{Notations}}\\
$[n]$   & $\cbrace{1, 2, \ldots, n}$, where $n \in \N$\\
$u_n \in \R^n$ & unit vector $(1, \ldots, 1) \in \R^n$ of all ones\\
$\inner{A, B}$ & Frobenius inner product between matrices $A$ and $B$\\
$A \circ B$ & Hadamard product between  matrices $A$ and $B$\\
$A \times B$ & Kronecker product between  matrices $A$ and $B$
\end{tabular}

\medskip

The different matrix products are defined as follows: Let  $A = (a_{ij})$ and $B = (b_{ij})$ be two matrices from $\R^{m \times n}$. The (Frobenius) inner product between $A$ and $B$ is defined by
\[
\inner{A, B} = \sum_{i=1}^m \sum_{j=1}^n a_{ij}b_{ij}.
\]
The Hadamard product $A \circ B$ of $A$ and $B$ is a matrix $C = c_{ij}$ from $\R^{m \times n}$ with elements $c_{ij} = a_{ij}b_{ij}$. The Kronecker product between matrices $A \in \R^{m \times n}$ and $B \in \R^{r \times s}$ is the ($mr \times ns$) block matrix
\[
A \times B = \begin{pmatrix}
a_{11} B & \cdots & a_{1n} B\\
\vdots & \ddots & \vdots\\
a_{m1} B & \cdots & a_{mn} B
\end{pmatrix}.
\]

\subsection{Equivalent Representations}

This section derives equivalent representations of max-linear and warped-linear functions. 

\subsubsection{Max-Linear Functions}

The following result reduces generalized linear functions to ordinary linear functions. 
\begin{lemma}\label{lemma:w'phi(x) = w_p'x}
Let $w \in \R^m$ and let $\phi: \R^d \rightarrow \R^m$ be a linear transformation. Then there is a $\hat{w} \in \R^d$ such that 
\[
w\T \phi(x) = \hat{w}\T x
\]
for all $x \in \R^d$. 
\end{lemma}

\begin{proof}
Since $\phi$ is a linear map, we can find a matrix $A \in \R^{m \times d}$ such that $\phi(x) = Ax$ for all $x \in \R^d$. Then we have
\[
w \T \phi(x) = w \T A x = \args{A\T w}{\T}x.
\] 
Setting $\hat{w} = A\T w$ completes the proof.
\end{proof}

\medskip

From Lemma \ref{lemma:w'phi(x) = w_p'x} follows that a max-linear function $f$ is of the equivalent form 
\[
f(x) = \max \cbrace{f_p(x) = \hat{w}_p \T x \,:\, p \in \S{P}},
\]
where $\hat{w}_p \in \R^d$ is the augmented weight vector of the $p$-th component function $f_p$. 

\subsubsection{Warped-Product Functions}

We present two equivalent definitions of warped-product functions. The first definition is based on rewriting the score $w \ast_p x$ as an inner product in the weight space. The second definition is based on rewriting $w \ast x$ as an inner product in the input space. 

\medskip

We begin with some preliminary work. Let $e^{k} \in \R^n$ denotes the $k$-th standard basis vector with elements
\[
e_i^k = \begin{cases}
1 & i = k\\
0 & i \neq k
\end{cases}
\] 
for all $i \in [n]$. In the following, the dimension of the standard basis vectors can be inferred from the context. 

\begin{definition}
Let $p = (p_1, \dots, p_L) \in \S{P}_{e, d}$ be a warping path with $L$ points $p_l = (i_l, j_l)$. Then 
\begin{equation*}
\Phi = \argsS{e^{i_1}, \ldots, e^{i_L}}{\tran} \in \R^{L \times e} ,\quad
\Psi = \argsS{e^{j_1}, \ldots, e^{j_L}}{\tran} \in \R^{L \times d}
\end{equation*}
is the pair of \emph{embedding matrices} induced by warping path $p$. 
\end{definition}
The $L$ rows of both embedding matrices are standard basis vectors from $\R^e$ and $\R^d$, respectively. The embedding matrices have full column rank due to the boundary and step condition of a warping path. Thus, we can regard the embedding matrices of warping path $p$ as injective linear maps $\Phi:\R^e \rightarrow \R^L$ and $\Psi:\R^d \rightarrow \R^L$ that embed every $w \in \R^e$ and every $x \in \R^d$ into $\R^L$ by matrix multiplications $\Phi w$ and $\Psi x$. 

\begin{definition}
The \emph{warping matrix} of warping path $p \in \S{P}_{e,d}$ is an ($e \times d$)-matrix $M_p = \args{m_{ij}^p}$ with elements
\begin{equation*}
m_{ij}^p = \begin{cases} 
1 & (i,j) \in p \\ 
0 & \text{otherwise} 
\end{cases}.
\end{equation*}
\end{definition}

\medskip

The next result relates the warping matrix of a warping path to the product of its embedding matrices. 
\begin{lemma}\label{lem:properties:embeddings}
Let $\Phi$ and $\Psi$ be the embedding matrices induced by warping path $p\in \S{P}_{e,d}$. Then the warping matrix $M_p\in \R^{e \times d}$ of warping path $p$ is of the form $M_p = \Phi \tran \Psi$.
\end{lemma}

\begin{proof}
\cite{Schultz2017}, Lemma A.3.
\end{proof}

\medskip

Now we are in the position to express the score $w \ast_p x$ as an inner product in the weight and in the input space. 
\begin{proposition}\label{prop:x*y=x'Wy}
Let $M_p$ be the warping matrix of warping path $p \in \S{P}_{e,d}$. Then we have
\[
w \ast_p x = w\T\args{M_p\phantom{\T}\!\! x} = \argsS{M_p\T w}{\T} x
\]
for all $w \in \R^e$ and all $x \in \R^d$.
\end{proposition}
\begin{proof}
It is sufficient to show that $w \ast_p x = w\T\args{M_p\phantom{\T}\!\! x}$. We assume that the warping path $p$ is given by $p = \args{p_1, \ldots, p_L}$ with points $p_l = \args{i_l, j_l}$. Consider the embedding matrices $\Phi = (\phi_{lk})$ and $\Psi = (\psi_{lk})$. The rows of $\Phi$ and $\Psi$ are standard basis vectors in $\R^e$ and $\R^d$, respectively. Then the elements of $\Phi$ and $\Psi$ are given by $\phi_{lk} = e^{i_l}_k$ and $\psi_{lk} = e^{j_l}_k$. We set $w' = \Phi w$ and $x' = \Psi x$. The elements of $w'$ and $x'$ are of the form
\begin{align*}
w'_l = \sum_{k=1}^e \phi_{lk} w_k = \sum_{k=1}^e e^{i_l}_k w_k = w_{i_l}\\
x'_l = \sum_{k=1}^d \psi_{lk} x_k = \sum_{k=1}^d e^{j_l}_k x_k = x_{j_l}
\end{align*}
for all $l \in [L]$.  From the definition of $w'$ and $x'$ together with Lemma \ref{lem:properties:embeddings} follows that
\[
w \ast_p x 
= \sum_{(i,j) \in p} w_i x_j 
= \sum_{l=1}^L w_{i_l} x_{j_l} 
= \sum_{l=1}^L w'_l x'_l 
= \args{\Phi w}{\T}\args{\Psi x}
= w\T\Phi\T\Psi x 
= w\T\args{M_p x}.
\]
This proves the assertion.
\end{proof}

\medskip

Suppose that $\S{Q} \subseteq \S{P}_{e,d}$ is a subset. From Prop.~\ref{prop:x*y=x'Wy} follows that a warped-product function $f$ can be equivalently written as  
\begin{align*}
f(x) 
&= \max \cbrace{w \T x_p \,:\, p \in \S{Q}}\\
&= \max \cbrace{w_p \T x \,:\, p \in \S{Q}},
\end{align*}
where $x_p = M_p x \in \R^e$ and $w_p = M_p\T w \in \R^d$ for all $p \in \S{Q}$.

\subsubsection{Elastic-Product Functions}

This section presents two equivalent definitions of elastic-product functions. As in the previous section, we show that both definitions are based on inner products in the weight and input space. 

\medskip
 
To express scores $W \otimes_p x$ as inner products in the weight space, we introduce the $p$-matrix of $x$. 
\begin{definition}
Let $p \in \S{P}_{d,e}$ be a warping path. The \emph{$p$-matrix} of $x \in \R^d$ is a ($d \times e$)-matrix $X_p = \args{x_{ij}^p}$ with elements
\[
x_{ij}^p = \begin{cases}
x_i & (i,j) \in p\\
0 & \text{otherwise}
\end{cases}.
\]
\end{definition}

\medskip

We obtain the $p$-matrix $X_p$ by embedding $x$ into a ($d \times e$)-dimensional zero-matrix along warping path $p$. We show that the transition from $x$ to $X_p$ is a linear map. 
\begin{lemma}\label{lemma:x->Xp}
The map 
\[
\phi: \R^d \rightarrow \R^{d \times e}, \quad x \mapsto M_p \circ \args{x \times u_e\T},
\]
is linear and satisfies $X_p = \phi(x)$
\end{lemma}
\begin{proof}
The Kronecker product and Hadamard product are linear maps. Since the linear maps are closed under composition, the map $\phi$ is linear. 

It remains to show that $X_p = \phi(x)$. The Kronecker product $X' = x \times u_e\T$ is a ($d \times e$)-matrix $X' = (x'_{ij})$ with elements $x_{ij}' = x_i$. The Schur product $X = M_p \circ X'$ is a ($d \times e$)-matrix $X = (x_{ij})$ with elements 
\[
x_{ij} = m_{ij}^p x'_{ij} = m_{ij}^p x_i.
\]
From the properties of the warping matrix $M_p$ follows that $x_{ij} = x_i$ if $(i,j) \in p$ and $x_{ij} = 0$ otherwise. Thus, we have $X_p = X$ and the proof is complete. 
\end{proof}

\medskip

To express scores $W \otimes_p x$ as inner products in the input space, we introduce the $p$-projection of $W$. 
\begin{definition}
Let $W \in \R^{d \times e}$ be a weight matrix and let $p \in \S{P}_{d,e}$ be a warping path. The \emph{$p$-projection} of $W$ is a vector $w_p \in \R^d$ defined by
\[
w_p = \args{M_p \circ W} u_e,
\]
where $M_p$ is the warping matrix of $p$ and $u_e \in \R^e$ is the vector of all ones. 
\end{definition}

\medskip

The next result shows that the scores $W \otimes_p x$ can be written as inner products in the weight and input space.

\begin{proposition}\label{prop:Wox = <W,Xp> = wp'x}
Let $W \in \R^{d \times e}$ be a matrix, let $x \in \R^d$ be a vector, and let $p \in \S{P}_{d,e}$ be a warping path. Then 
\[
W \otimes_p x = \inner{W, X_p} = w_p\T x,
\]
where $X_p$ is the $p$-matrix of $x$ and $w_p = \psi_p(W)$ is the $p$-projection of $W$.
\end{proposition}

\begin{proof}
It is sufficient to show the assertions $W \otimes_p x = \inner{W, X_p}$ and $W \otimes_p x = w_p\T x$. The first assertion $W \otimes_p x = \inner{W, X_p}$ follows from
\begin{align*}
\inner{W, X_p} 
&= \sum_{i=1}^d \sum_{j=1}^e w_{ij}^{\phantom{p}}x_{ij}^p
= \sum_{(i,j) \in p} w_{ij}x_i = W \otimes_p x.
\end{align*}

We show the second assertion $W \otimes_p x = w_p\T x$. Let $M_p = \args{m_{ij}^p}$ be the warping matrix of path $p$. For every $i \in [d]$ let $m_i^p$ and $w_i$ denote the $i$-th row of $M_p$ and $W$, respectively. Then the elements $w_i^p$ of the $p$-projection $w_p = \args{w_1^p, \ldots, w_d^p}$ are of the form
\[
w_i^p = (m_i^p \circ w_i)\T u_e = \sum_{j=1}^e m_{ij}^p w_{ij}.
\]
Then the second assertion follows from
\[
W \otimes_p x = \sum_{(i,j) \in p} w_{ij}x_i = \sum_{i=1}^d \sum_{j=1}^e m_{ij}^p w_{ij}x_i = \sum_{i=1}^d w_i^p x_i = w_p\T x.
\]
\end{proof}

\medskip

Suppose that $\S{Q} \subseteq \S{P}_{d, e}$ is a subset. From Prop.~\ref{prop:Wox = <W,Xp> = wp'x} follows that an elastic-product function $f$ can be equivalently written as  
\begin{align*}
f(x) 
&= \max \cbrace{\inner{W, X_p} \,:\, p \in \S{Q}}\\
&= \max \cbrace{w_p \T x \,:\, p \in \S{Q}},
\end{align*}
where $X_p$ is the $p$-matrix of $x$ and $w_p$ is the $p$-projection of $W$ for all $p \in \S{Q}$.

\subsection{Proof of Theorem \ref{theorem:warped=max-linear}}

\subsubsection*{Proof of $\S{W}_d \subseteq \S{L}_d$}
Let $f \in \S{W}_d$ be a warped-product function of elasticity $e$ in $\S{Q} \subseteq \S{P}_{e,d}$. Then there is a weight sequence $w\in \R^e$ such that $f$ is of the form 
\[
f(x) = \max \cbrace{ w \ast_p x\,: \, p \in \S{Q}}.
\]
From Prop.~\ref{prop:x*y=x'Wy} follows that 
\[
w \ast_p x = w\T M_p\,x, 
\]
where $M_p$ is the warping matrix of $p$. Consider the linear transformation $\phi_p(x) = M_p\, x$. Then we can equivalently express $f$ by
\[
f(x) = \max \cbrace{ f_p(x)\,: \, p \in \S{Q}},
\]
where the components $f_p(x) = w \T \phi_p(x)$ are generalized linear functions indexed by $p \in \S{Q}$. This shows that $f \in \S{L}_d$ and therefore $\S{W}_d \subseteq \S{L}_d$. 
\qed

\subsubsection*{Proof of $\S{E}_d \subseteq \S{L}_d$}
Let $f \in \S{E}_d$ be a warped-product function of elasticity $e$ in $\S{Q} \subseteq \S{P}_{d, e}$. Then there is a weight matrix $W\in \R^{d \times e}$ such that $f$ is of the form 
\[
f(x) = \max \cbrace{ W \otimes_p x\,: \, p \in \S{Q}}.
\]
Let $u_e = (1, \ldots, 1) \in \R^e$ be the vector of all ones. The map $\phi_p(x)$ defined in Lemma \ref{lemma:x->Xp} is linear. From  Prop.~\ref{prop:Wox = <W,Xp> = wp'x} follows that $f$ can be equivalently written as 
\[
f(x) = \max \cbrace{ f_p(x)\,: \, p \in \S{Q}},
\]
where the components $f_p(x) = \inner{W, \phi_p(x)}$ are generalized linear functions indexed by $p \in \S{Q}$. This shows that $f \in \S{L}_d$ and therefore $\S{E}_d \subseteq \S{L}_d$. 
\qed

\subsubsection*{Proof of $\S{L}_d \subseteq \S{E}_d$}

Suppose that the index set is given by $\S{P} = [c]$ for some $c > 0$. Let $f \in \S{L}_d$ be a max-linear function of the form 
\[
f(x) = \max \cbrace{w_p \T \phi_p(x) \,:\, p \in [c]},
\]
where $w_p \in \R^m$ for all $p \in [c]$. From Lemma \ref{lemma:w'phi(x) = w_p'x} follows that we can equivalently rewrite $f$ as 
\[
f(x) = \max \cbrace{\hat{w}_p \T x \,:\, p \in [c]},
\]
where $\hat{w}_p \in \R^d$ for all $p \in [c]$. Let 
\[
A = \argsS{\hat{w}_1, \ldots, \hat{w}_c}{\T}
\]
be the matrix whose rows are the weight vectors $\hat{w}_p \in \R^d$ of the components $f_p$. From Lemma \ref{lemma:A->W} follows that there is a weight matrix $W \in \R^{d \times e}$ and a subset $\S{Q} \subseteq \S{P}_{d \times e}$ such that $W \otimes_p x = \hat{w}_p \T x$. Hence, $f \in \S{E}_d$ and the proof is complete. 
\qed

\bigskip 

To complete the proof of  $\S{L}_d \subseteq \S{E}_d$, we need to show Lemma \ref{lemma:A->W}. The main steps of the proof are illustrated by Example \ref{ex:proof-lemma:A->W}.

\begin{lemma}\label{lemma:A->W}
Let $A \in \R^{c \times d}$ be a matrix with $c$ rows $a_i \in \R^d$ for all $i \in [c]$. Then there is a matrix $W \in \R^{d \times e}$, a subset $\S{Q} \subseteq \S{P}_{d,e}$ of $c$ warping paths, and a bijection $\pi:\S{Q} \rightarrow [c]$ such that
\[
\psi_p(W) = a_{\pi(p)}, 
\]
for all $p \in \S{Q}$. 
\end{lemma} 

\begin{proof}
Let $A' = \args{a_{ij}'} \in \R^{c \times d}$ be the matrix with elements
\[
a_{ij}' = \begin{cases}
a_{ij} - a_{i-1,j} & j = 1 \text{ and } i > 1\\
a_{ij} - a_{i+1,j} & j = d \text{ and } i < c\\
a_{ij} & \text{otherwise}
\end{cases}.
\]
Consider the matrix $B = A'{\T} \in \R^{d \times c}$ with elements $b_{ij} = a_{ji}'$. Let $n = 2c-1$. We define a matrix $C = (c_{ij}) \in \R^{d \times n}$ with elements
\[
c_{ij} = \begin{cases}
b_{ik} & j = 2(k-1)+1\\
0 & i \in \cbrace{1,d} \; \wedge \; {j\text{ mod }2} = 0\\
* & \text{otherwise}
\end{cases},
\]
where $*$ denotes the \emph{don't care} symbol. We set the elasticity $e = 2(c-1) + d$. Consider the weight matrix $W \in \R^{d \times e}$ with elements
\[
w_{ij} = \begin{cases}
c_{ik} & j = k + i-1\\
* & \text{otherwise}
\end{cases}.
\]
Let $\S{Q} \subseteq \S{P}_{d,e}$ be the subset consisting of all warping paths $p$ such that $(i,j) \in p$ implies $w_{ij} \neq *$. By construction, the warping paths $p$ in $\S{Q}$ are in one-to-one correspondence with the rows $a_i$ of matrix $A$ such that $w_p\T x = a_i\T x$, where $w_p = \psi_p(W)$ is the $p$-projection of $W$. This completes the proof.
\end{proof}

\medskip

\begin{example}\label{ex:proof-lemma:A->W}
Let $A = \args{a_{ij}}\in \R^{c \times d}$ be a matrix with $c = 3$, let $d = 4$. The matrix $A' = \args{a_{ij}'}$ is of the form
\[
A' = \begin{pmatrix}
a_{11} & a_{12} & a_{13} & a_{14}' \\
a_{21}' & a_{22} & a_{23} & a_{24}' \\
a_{31}' & a_{32} & a_{33} & a_{34} 
\end{pmatrix},
\] 
where 
\[
a_{21}' = a_{21} - a_{11}, \quad a_{31}' = a_{31}-a_{21}, \quad a_{14}' = a_{14} - a_{24}, \quad a_{24}' = a_{24}-a_{34}.
\]
Taking the transpose of $A'$ gives
\[
B = \begin{pmatrix}
a_{11} & a_{21}' & a_{31}'\\
a_{12} & a_{22} & a_{32} \\
a_{13} & a_{23} & a_{33} \\
a_{14}' & a_{24}' & a_{34}
\end{pmatrix}.
\]
Inserting zeros and don't care symbols gives
\[
C = \begin{pmatrix}
a_{11} & 0 & a_{21}' & 0 & a_{31}'\\
a_{12} & * & a_{22} & * & a_{32} \\
a_{13} & * & a_{23} & * & a_{33} \\
a_{14}' & 0 & a_{24}' & 0 & a_{34}
\end{pmatrix}.
\]
Finally, shifting the rows of $C$ yields the weight matrix
\[
W = 
\begin{pmatrix}
a_{11} & 0 & a_{21}' & 0 & a_{31}' & * & * & *\\
* & a_{12} & * & a_{22} & * & a_{32} & * & *\\
* & * & a_{13} & * & a_{23} & * & a_{33} & *\\
* & * & * & a_{14}' & 0 & a_{24}'& 0 & a_{34}
\end{pmatrix}.
\]
\end{example}

\subsection{Proof of Proposition \ref{prop:W'_d = L_d}}

Observe that 
\[
\phi_0 : \R^d \times \R^{d+2}, \quad x \mapsto 
\begin{pmatrix}
0\\[-0.4ex]
x\\[-0.4ex]
0
\end{pmatrix}
\] 
is a linear map. Since the composition of linear maps is linear, the relationship $\S{W}'_d \subseteq \S{L}_d$ follows from the first part of the proof of Theorem \ref{theorem:warped=max-linear}. It remains to show that $\S{L}_d \subseteq \S{W}'_d$. 

Let $\S{P} = [c]$ be the index set for some $c > 0$. In accordance with Lemma \ref{lemma:w'phi(x) = w_p'x}, let $f \in \S{L}_d$ be a max-linear function of the form 
\[
f(x) = \max \cbrace{\hat{w}_p \T x \,:\, p \in [c]},
\]
where $\hat{w}_p \in \R^d$ for all $p \in [c]$. Let 
\[
A = \argsS{\hat{w}_1, \ldots, \hat{w}_c}{\T}
\]
be the matrix whose rows are the weight vectors $\hat{w}_p \in \R^d$ of the components $f_p$. Let $e = cd$ and let 
\[
w = \argsS{\hat{w}_1\T, \ldots, \hat{w}_c\T}{\T} \in \R^e
\]
be the vector obtained by concatenating the columns $\hat{w}_p$ of $A$. Consider the matrix $M = (m_{ij}) \in \cbrace{0,1}^{e \times d+2}$. Suppose that $m_{ij} = 1$ if and only if the following condition is satisfied:
\begin{align*}
& \Big(1 \leq i \leq (c-1)d+1 \; \land \; j = 1 \Big)\\
\; \lor \;&  \Big(d \leq i \leq cd \; \land \; j = d+2 \Big)\\
\; \lor \;&  \Big(i = j-1 + kd \; \land \; 2 \leq j \leq d+1 \; \land 0 \leq k \leq (c-1)d\Big).
\end{align*}
Let $\S{Q} \subseteq \S{P}_{e, d+2}$ be the subset consisting of all warping paths $p$ such that $m_{ij} = 1$ for all points $(i,j) \in p$. Then the set $\S{Q}$ consists of exactly $c$ warping paths. Suppose that $M_p \in \cbrace{0,1}^{e \times d+2}$ is the warping matrix of $p \in \S{Q}$. By $M_p' \in \cbrace{0,1}^{e \times d}$ we denote the matrix obtained from $M_p$ by removing the first and last column. Then we have
\[
w \ast_p \phi_0(x) = w\T M_p \phi_0(x) = w\T M'_p x = \argsS{M'_p{\T}w}\T x.
\]
By construction of $\S{Q}$ we find that $M'_p{\T}w = \hat{w}_p $ giving $w \ast_p \phi_0(x) = \hat{w}_p\T x$. Hence, $f \in \S{W}'_d$ and the proof is complete
\qed

\commentout{
\begin{example}\label{ex:W_d!=L_d}
We assume that $d = 1$. Suppose that $a > 0$ is a scalar. Consider the max-linear function
\[
f:\R \rightarrow \R, \quad x \mapsto \max\cbrace{a x, -a x}
\]
with two components both with zero bias. Then we have $f(x) \geq 0$ for all $x \in \R$. 

Let $w \in \R^e$ be a weight sequence. Then the set $\S{P}_{e,1}$ consists of exactly one warping path. Suppose that $\alpha = (w_1 + \cdots + w_e) \in \R$. Then we have $g(x) = w \ast x = \alpha x$ for all $x \in \R$. The function $g: \R \rightarrow \R$ is either zero or a bijective function. In either case, we find that $f \neq q$. 
\end{example}
}

\subsection{Proof of Proposition \ref{prop:L C diff C lin}}

\begin{proposition}
$\S{L}_d \subsetneq \Delta(\S{L}_d) = \linear\args{\S{L}_d}$.
\end{proposition}

\subsubsection*{Proof of $\S{L}_d \subseteq \Delta(\S{L}_d)$}
Let $f \in \S{L}_d$ be a max-linear function. From Lemma \ref{lemma:w'phi(x) = w_p'x} follows that $f$ can be written as 
\[
f(x) = \max \cbrace{w_p \T x \,:\, p \in S{P}},
\]
where $w_p \in \R^d$ for all $p \in \S{P}$. Observe that the constant function $g(x) = 0$ is contained in $\S{L}_d$. Hence, $f = f-g = \in \Delta(\S{L}_d)$. This proves the assertion. 
\qed
 
\subsubsection*{Proof of $\S{L}_d \neq \Delta(\S{L}_d)$}
Max-linear functions are convex, because linear functions are convex and convexity is closed under max-operation. Let $f \in \S{L}_d$ be non-linear and let $g \in \S{L}_d$ be the constant function $g(x) = 0$. Then the function $g-f$ is an element of $\Delta(\S{L}_d)$. But the function $-f = g-f$ is not contained in $\S{L}_d$, because $-f$ is non-convex. This completes the proof. 
\qed

\subsubsection*{Proof of $\Delta(\S{L}_d) = \linear\args{\S{L}_d}$}

The relationship $\Delta(\S{L}_d) \subseteq \linear\args{\S{L}_d}$ follows directly from the definition of both sets. We show $\linear\args{\S{L}_d} \subseteq \Delta(\S{L}_d)$. Let $f \in \linear\args{\S{L}_d}$. Then there are max-linear function $f_1, f_2 \in \S{L}_d$ and scalars $\lambda_1, \lambda_2 \in \R$ such that $f = \lambda_1 f_1 + \lambda_2 f_2$. We distinguish between the following cases:
\begin{enumerate}
\item $\lambda_1 \geq 0, \lambda_2 \geq 0$: Then $h_i = \lambda_i f_i$ is a max-linear function for $i \in \cbrace{1,2}$. In addition, the sum $f = h_1 + h_2$ is max-linear. From the first part of this proof follows that $f \in \Delta(\S{L}_d)$. 
\item $\lambda_1 \geq 0, \lambda_2 < 0$: Then $h_1 = \lambda_1 f_1$ and $h_2 = \abs{\lambda_2} f_2$ are max-linear functions. Hence, $f = h_1 - h_2$ is an element of $\Delta(\S{L}_d)$. 
\item $\lambda_1 < 0, \lambda_2 \geq 0$: Symmetric to the second case. 
\item $\lambda_1 < 0, \lambda_2 < 0$: Then $h_1 = \abs{\lambda_1} f_1$ and $h_2 = \abs{\lambda_2} f_2$ are max-linear functions and therefore $h = h_1 + h_2$ is a max-linear function with $f = -h$. We find that $f = 0-h$ is contained in $\Delta(\S{L}_d)$. 
\end{enumerate}
From all four cases follows $f \in \Delta(\S{L}_d)$ and therefore $\linear\args{\S{L}_d} \subseteq \Delta(\S{L}_d)$. This completes the proof.

\subsection{Subgradients}\label{sec:basic-definitions}

The \emph{subdifferential} of a convex function $f:\R^n \rightarrow \R$ at $x \in \R^n$ is the set 
\[
\partial f(x) = \cbrace{\xi \in \R^d \,:\, f(y) \geq f(x) + \xi\T(y-x)}.
\]
The elements of the subdifferential $\partial f(x)$ are the \emph{subgradients} of $f$ at $x$. 

\medskip

Let $f_\theta: \S{X} \rightarrow \R$ be a max-linear function with components $f_1, \ldots, f_c:\S{X} \rightarrow \R$ of the form $f_p(x) = w_p\T\phi_p(x)$, where $\phi_p : \R^d \rightarrow \R^m$ is a linear transformation such that $\phi_p(\S{X}) \subseteq \cbrace{1} \times \R^{m-1}$. The parameters of $f_\theta$ are summarized by the vector
\[
\theta = \begin{pmatrix}
w_1\\
\vdots\\
w_c
\end{pmatrix} \in \R^q,
\]
where $q = c \cdot m$. We regard all vectors $u \in \R^q$ as a stack of $c$ vectors $u[1], \ldots, u[c] \in \R^m$, briefly called \emph{segments} henceforth. Thus, the segments of the parameter vector $\theta$ are given by $\theta[p] = w_p$ for all $p \in [c]$. For every $p \in [c]$ we introduce the \emph{$p$-inflation function} 
\[
\psi_p : \R^m \rightarrow \R^q, \quad u \mapsto u_p 
\]
such that
\[
u_p[k]= \begin{cases}
u & k = p\\
0 & k \neq p
\end{cases}
\]
The $p$-th segment $u_p[p]$ of $u_p$ coincides with $u$ and all other segments of $u_p$ are zero vectors. We call $u_p$ the \emph{$p$-inflation} of $u$. 

\medskip

Next, we rewrite the regularized loss as a function of the parameters $\theta$. The regularized loss for training example $z = (x, y)$ is of the form
\[
h_\theta(x) = \ell\args{y, f_\theta(x)} + \lambda \rho(\theta).
\]
For every $p \in \S{P}$ let $x_p = \psi_p(\phi_p(x))$ denote the $p$-inflation of $\phi_p(x)$. Then we define the following functions of $\theta$ parametrized by $z$: 
\begin{align*}
f_p(\theta; z) 		&= \theta\T x_p = w_p\T \phi_p(x) = f_p(x)\\
f(\theta; z) 		&= \max\cbrace{f_p(\theta; z) \,:\, p \in [c]} = f_\theta(x)\\
\ell(\theta; z) 		&= \ell\args{f(\theta; z); z} = \ell\args{y, f_\theta(x)}\\
\rho(\theta; z)		&= \rho(\theta)\\
h(\theta; z)			&= \ell(\theta; z) + \lambda\rho(\theta; z).
\end{align*}
Note that not every function depends on $z$ and not every function depends on all components of $z$. In those cases, the dependence on $z$ is included for the sake of conformity. We prove the following proposition.

\begin{proposition}\label{prop:subgradient}
Let $z =(x, y) \in \S{X} \times \S{Y}$ be a training example, let $f_\theta$ be a max-linear function with parameter $\theta \in \R^q$, and let $f_p(x) = w_p\T\phi_p(x)$ be an active component of $f$ at $x$. Suppose that the loss $\ell(\theta; z)= \ell\args{\hat{y}; z}$ is convex as a function of $\hat{y} = f(\theta; z)$ and the regularization function $\rho(\theta; z)$ is convex as a function of $\theta$. Then the regularized loss 
\[
h(\theta; z)= \ell(\theta; z) + \lambda\rho(\theta; z)
\]
is convex and 
\[
\alpha \cdot \psi_p(\phi_p(x)) +  \lambda \xi \in \partial_\theta h_x(\theta)
\]
for every subgradient $\alpha \in \partial_\theta \ell_y(\theta) \subseteq \R$ and for every subgradient $\xi \in \partial_\theta \rho(\theta)$.
\end{proposition}

\medskip

Before we can prove Prop.~\ref{prop:subgradient}, we need some auxiliary results.

\begin{lemma}\label{lemma:convex}
A max-linear function is convex. 
\end{lemma}

\begin{proof}
Linear function are convex. Since convexity is closed under max-operations, a max-linear function is convex. 
\end{proof}

\begin{lemma}\label{lemma:subgradient-of-f}
Let $z = (x, y) \in \S{X} \times \S{Y}$ and let $f_\theta: \S{X} \rightarrow \R$ be a max-linear function with parameter $\theta \in \R^q$. Suppose that $f_p  = w_p{\T} \phi_p(x)$ is an active component of $f$ at $x$. Then
\[
x_p \in \partial_\theta f(\theta; z),
\]
where $x_p = \psi_p(\phi_p(x))$ is the $p$-inflation of $\phi_p(x)$.
\end{lemma}

\begin{proof}
The function $f(\theta; z)$ is max-linear and therefore convex by Lemma \ref{lemma:convex}. Hence, the subdifferential $\partial_\theta f(\theta; z)$ exists and is non-empty by \cite{Bagirov2014}, Theorem 2.27.

The component $f_p(\theta; z) = \theta\T x_p$ is linear with gradient $\nabla_\theta f_p(\theta; z) = x_p$. Suppose that $\theta' \in \R^q$ is an arbitrary parameter vector. Then the following system of equations holds:
\begin{align*}
f(\theta'; z) &\geq f_p(\theta'; z) \\
f(\theta; z)  &= f_p(\theta; z)\\
f_p(\theta'; z) &= f_p(\theta; z) + x_p\T (\theta' - \theta) 
\end{align*}
The first inequality holds, because $f_p(\theta';z)$ is a component of the max-linear function $f(\theta';z)$. The second equation holds, because $f_p(\theta;z)$ is an active component of $f(\theta;z)$ at $\theta$. Finally, the last equation holds by the properties of a linear function. Combining the equations yields
\[
f_x(\theta') \geq f_p(\theta'; x) = f_p(\theta) + x_p\T (\theta' - \theta)
\]
for all $\theta' \in \Theta$. This shows that $x_p \in \partial_\theta f_x(\theta)$.
\end{proof}

\subsubsection*{Proof of Prop.~\ref{prop:subgradient}}

\begin{proof}
Let $x_p = \psi_p(\phi_p(x))$ be the $p$-inflation of $\phi_p(x)$. We first ignore the regularization term. Suppose that 
$h_0(\theta) = \ell(f(\theta; z); z)$. The loss $\ell(f(\theta; z); z)$ is convex as a function of $f(\theta; z)$ by assumption and the max-linear function $f(\theta; z)$ is convex by Lemma \ref{lemma:convex}. As convex functions, $\ell$ and $f$ are subdifferentially regular by \cite{Bagirov2014}, Theorem 3.13. Hence, we can invoke \cite{Bagirov2014}, Theorem 3.20 and obtain that $h_0$ is subdifferentially regular such that
\[
\partial_\theta h_0(\theta) = \conv \cbrace{\partial_\theta f(\theta; z) \T \partial_\theta \ell(f(\theta; z); z)}.
\]
From Lemma \ref{lemma:subgradient-of-f} follows that $x_p \in \partial_\theta f(\theta; z)$. This shows that $\alpha \cdot x_p \in \partial_\theta h_0(\theta)$ for every $\alpha \in \partial_\theta \ell(f(\theta; z); z)$. 

\medskip

Next, we include the regularization function $\rho(\theta; z)$. Since $\rho$ is convex by assumption, \cite{Bagirov2014}, Theorem 3.13 yields that $\rho$ is subdifferentially regular. Then from \cite{Bagirov2014}, Theorem 3.16 follows that the function $h(\theta; z) = h_0(\theta) + \rho(\theta; z)$ is subdifferentially regular and
\[
\partial_\theta h(\theta; z) = \partial_\theta h_0(\theta)  + \partial_\theta \rho(\theta; z).
\]
Together with the first part of this proof, we obtain the assertion that $\alpha \cdot x_p + \xi \in \partial_\theta h(\theta; z)$ for every $\alpha \in \partial_\theta \ell(f(\theta; z); z)$ and for every $\xi \in \partial_\theta \rho(\theta; z)$. 
\end{proof}
\end{appendix}

\end{document}

%% file: config.tex
\usepackage[english]{babel} 
\usepackage[protrusion=true,expansion=true]{microtype} 
\usepackage{amsmath,amsfonts,amsthm,amssymb,color,epsfig,fullpage,mathabx} 
\usepackage{verbatim,fancyvrb}
\usepackage{alltt}
\usepackage{tikz}
\usepackage{enumerate}
\usepackage{caption}
\usepackage{subfigure}
\usepackage{stmaryrd}
\usepackage{afterpage}
\usepackage{url}

\newtheorem{definition}{Definition}[section]
\newtheorem{theorem}[definition]{Theorem}

\newtheorem{proposition}[definition]{Proposition}
\newtheorem{lemma}[definition]{Lemma}
\newtheorem{example}[definition]{Example}
\newtheorem{remark}[definition]{Remark}
\newtheorem{notation}[definition]{Notation}
\newtheorem{assumption}[definition]{Assumption}

\def\N{{\mathbb N}}

\def\R{{\mathbb R}}

\DeclareMathOperator*{\argmax}{argmax}

\DeclareMathOperator*{\conv}{conv}

\DeclareMathOperator*{\linear}{lin}

\newcommand*{\tran}{^{\mkern-1.5mu\mathsf{T}}}

\newcommand{\commentout}[1]{}

\newcommand{\abs}[1]{\mathop{\left\lvert #1 \right\rvert}} 
\newcommand{\args}[1]{\mathop{\left( #1 \right)}} 
\newcommand{\inner}[1]{\mathop{\left\langle #1 \right\rangle}}

\newcommand{\cbrace}[1]{\mathop{\left\{ #1 \right\}}}

\newcommand{\argsS}[2]{\mathop{\left( #1 \right)#2}} 

\newcommand{\normS}[2]{\mathop{\left\lVert #1 \right\rVert#2}}

\DeclareMathOperator{\id}{id}                    				
\newcommand{\T}{^{\mkern1.5mu\mathsf{T}}}     

\renewcommand{\S}[1]{{\mathcal{#1}}}           	

\renewenvironment{cases}{%
\left\{\begin{array}{c@{\quad : \quad}l}}%
{%
\end{array}\right.}

\newcommand{\atab}[1]{\hspace*{#1em}}

\newcounter{algorithm_counter}
\newenvironment{algorithm}[1]{
\refstepcounter{algorithm_counter}
\setlength{\parindent}{0\parindent}
\vspace{2ex}
\begin{minipage}{\textwidth}
\rule{\textwidth}{\arrayrulewidth}\\
\begin{footnotesize}
\textbf{Algorithm:} #1 \\
\rule[+1.5ex]{\textwidth}{\arrayrulewidth}

\vspace{-1.5ex}

}%
{
\\[-1.5ex]
\rule{\textwidth}{\arrayrulewidth}
\end{footnotesize}
\end{minipage}
\setlength{\parindent}{\parindent}
}

{
\\[-1.5ex]
\rule{\textwidth}{\arrayrulewidth}
\end{footnotesize}
\end{minipage}
\setlength{\parindent}{\parindent}
}